\documentclass[12pt]{article}
\usepackage{amsmath}
\usepackage{times}
\usepackage{graphicx}
\usepackage{color}
\usepackage{multirow}
\usepackage[authoryear]{natbib}
\usepackage{rotating}
\usepackage{bbm}
\usepackage{latexsym}

\newcommand{\captionfonts}{\normalsize}

\makeatletter
\long\def\@makecaption#1#2{%
  \vskip\abovecaptionskip
  \sbox\@tempboxa{{\captionfonts #1: #2}}%
  \ifdim \wd\@tempboxa >\hsize
    {\captionfonts #1: #2\par}
  \else
    \hbox to\hsize{\hfil\box\@tempboxa\hfil}%
  \fi
  \vskip\belowcaptionskip}
\makeatother

\usepackage{amsthm}
\usepackage{xcolor}
\usepackage{amsfonts}
\usepackage{amsmath}
\usepackage{graphicx}
\usepackage{tikz}
\usepackage{algorithm}
\usepackage{algorithmic}
\usepackage{hyperref}
\usepackage{subfigure}
\newcommand{\comment}[1]{}

\newtheorem{proposition}{Proposition}
\newtheorem{remark}{Remark}
\newtheorem{theorem}{Theorem}
\newtheorem{corollary}{Corollary}
\newtheorem{lemma}{Lemma}

\begin{document}

\hspace{13.9cm}1

\ \vspace{25mm}\\

{\LARGE Relations among Some Low Rank Subspace\\ Recovery Models}

\ \\
{\bf \large Hongyang Zhang$^{\dag}$,  Zhouchen Lin$^{\dag}$,  Chao Zhang$^{\dag}$\footnotemark[1], Junbin Gao$^{\ddag}$}\\
{$^{\dag}$Key Laboratory of Machine Perception (MOE), School of EECS, Peking University, Beijing 100871, China.}\\
{$^{\ddag}$School of Computing and Mathematics, Charles Sturt University, Bathurst, NSW 2795, Australia.}\\
\footnotetext[1]{Corresponding author.}

\thispagestyle{empty}
\markboth{}{NC instructions}
\ \vspace{-10mm}\\
%
\begin{center} {\bf Abstract} \end{center}
Recovering intrinsic low dimensional subspaces from data
distributed on them is a key preprocessing step to many applications. In
recent years, there has been a lot of work that models subspace
recovery as low rank minimization problems. We find that some
representative models, such as Robust Principal Component Analysis
(R-PCA), Robust Low Rank Representation (R-LRR), and Robust Latent
Low Rank Representation (R-LatLRR), are actually deeply connected.
More specifically, we discover that once a solution to one of the
models is obtained, we can obtain the solutions to
other models in \emph{closed-form} formulations. Since R-PCA is the simplest, our discovery makes it
the center of low rank subspace recovery models. Our work has two
important implications. First, R-PCA has a solid theoretical foundation. Under certain conditions, we could find better solutions to these low rank models at overwhelming probabilities, although these models are non-convex. Second, we can obtain
significantly faster algorithms for these models by solving R-PCA
first. The computation cost can be further cut by applying low
complexity randomized algorithms, e.g., our novel $\ell_{2,1}$
filtering algorithm, to R-PCA. Experiments verify the advantages of our algorithms over other state-of-the-art ones that are based on the alternating direction
method.

\section{Introduction}
Subspaces are the most commonly assumed structure for high
dimensional data due to their simplicity and effectiveness. For
example, motion~\citep{Tomasi1992}, face~\citep{Belhumeur:Fisherface,belhumeur1998set,Basri2003}, and
texture~\citep{Ma2} data have been known to be well characterized
by low dimensional subspaces. There has been a lot of effort on
robustly recovering the underlying subspaces of data. The most
widely adopted approach is Principal Component Analysis (PCA).
Unfortunately, PCA is known to be fragile to large noises or
outliers. So much work has been devoted to improving the
robustness of
PCA~\citep{gnanadesikan1972robust,huber2011robust,fischler1981random,Torre,ke2005robust}, among
which the Robust PCA (R-PCA)~\citep{Wright:RPCA,chandrasekaran2011rank,Candes} is probably
the only one with theoretical guarantees. \cite{Candes,chandrasekaran2011rank,Wright:RPCA} proved that under certain conditions
the ground truth subspace can be exactly recovered with an
overwhelming probability. Later work~\citep{hsu2011robust} gave a justification of R-PCA in the case where the spatial pattern of
the corruptions is deterministic.

Although R-PCA has found wide applications, such as video
denoising, background modeling, image alignment, photometric
stereo, and texture representation
\cite[see e.g.,][]{Wright:RPCA,Torre,Ji,Peng,ZhangZ}, it only aims at
recovering a single subspace that spans the whole data. To
identify finer structure of data, the multiple subspaces
recovery problem should be considered, which aims at clustering
data according to the subspaces they lie in. This problem has
attracted a lot of attention in recent
years~\citep{vidal2011subspace}. Rank minimization methods account
for a large class of subspace clustering algorithms, where rank is
connected to the dimensions of subspaces. Representative rank
minimization based methods include Low Rank Representation
(LRR)~\citep{LiuG3,LiuG1}, Robust Low Rank Representation
(R-LRR)~\citep{Wei,Vidal:LRSC} \footnote{Note that
\cite{Wei} and \cite{Vidal:LRSC} called R-LRR as Robust Shape
Interaction (RSI) and Low Rank Subspace Clustering (LRSC),
respectively. The two models are essentially the same, only
differing in the optimization algorithms. In order to remind the
readers that they are both robust versions of LRR by using a
denoised dictionary, in this paper we call them Robust Low Rank
Representation (R-LRR) instead.}, Latent Low Rank Representation
(LatLRR)~\citep{LiuG2,Zhang:Counterexample} and its robust version
(R-LatLRR)~\citep{Zhang:RobustLatLRR}. Nowadays, subspace clustering algorithms, including these low rank
methods, have been widely applied, e.g., to motion
segmentation~\citep{Gear,Costeira,Vidal2004CVPR,Yan,Rao}, image
segmentation~\citep{Yang2008CVIU,Cheng}, face
classification~\citep{Ho2003,Vidal2005PAMI,LiuG3,LiuG1}, and system
identification~\citep{Vidal2003CDC,ZhangC,Paoletti2007}.

\subsection{Our Contributions}
In this paper, we show that some of the low rank subspace recovery models are actually deeply connected, even though they were proposed independently and targeted different problems (single or multiple subspaces recovery). Our
discoveries are based on a characteristic of low rank recovery
models. Namely, they may have closed-form solutions. Such a
characteristic has not been found in sparsity based models, e.g.,
Sparse Subspace Clustering~\citep{Elhamifar}.

There are two main contributions of this paper:
\begin{itemize}
\item We find a close relation between
R-LRR~\citep{Wei,Vidal:LRSC} and R-PCA~\citep{Wright:RPCA,Candes},
showing that, surprisingly, their solutions are mutually
expressible. Similarly,
R-LatLRR~\citep{Zhang:RobustLatLRR} and R-PCA are closely
connected too. Namely, their solutions are also mutually
expressible. Our analysis allows an arbitrary regularizer for the
noise term. \item Since R-PCA is the simplest low rank recovery model, our analysis naturally positions R-PCA at the center of existing low rank recovery models. In particular, we propose to first apply R-PCA to the data and then use the solution of R-PCA to obtain the solution for other models. This approach has two important implications. First, although R-LRR and
R-LatLRR are non-convex problems, under certain conditions we can
obtain better solutions with an overwhelming probability. Namely, if the noiseless data are sampled from a union of independent subspaces and the dimension of the subspace containing the union of subspaces is much smaller than the dimension of the ambient space, we are able to recover exact subspaces structure as long as the noises are sparse (even the magnitudes of noise are arbitrarily large). Second, solving R-PCA is much faster than solving other models. The computation cost could be further cut if we solve
R-PCA by randomized algorithms. For example, we propose the
$\ell_{2,1}$ filtering algorithm to solve R-PCA when the noise term
uses $\ell_{2,1}$ norm (see Table~\ref{table: notations} for
definition). Experiments verify the significant advantages of our
algorithms.
\end{itemize}

The remainder of this paper is organized as follows. Section
\ref{section:previous works} reviews the representative low rank
models for subspace recovery. Section \ref{section: Theoretical
Relationships between Low Rank Models} gives our theoretical
results, i.e., the inter-expressibility among the solutions of
R-PCA, R-LRR, and R-LatLRR. In Section \ref{section: proofs of the
main results}, we present detailed proofs of our theoretical
results. Section \ref{section:applications of theoretical
analysis} gives two implications of our theoretical analysis,
i.e., better solutions and faster algorithms. We show the
experimental results on both synthetic and real data in Section
\ref{section:experiments}. Finally, we conclude the paper.

\section{Related Work}
\label{section:previous works} In this section, we review a number
of existing low rank models for subspace recovery.

\subsection{Notations and Naming Conventions}
Before start, we define some notations that we will use. Table
\ref{table: notations} summarizes the main notations that will
appear in this paper.
\makeatletter\def\@captype{table}\makeatother
\begin{table}
\caption{Summary of main notations used in this paper.}
\label{table: notations}
\begin{center}
\begin{tabular}{c|l}
\hline
Notations & Meanings\\
\hline
Capital letter & A matrix.\\
$m$, $n$ & Size of the data matrix $M$.\\
$n_{(1)}$, $n_{(2)}$ & $n_{(1)}=\max\{m,n\}$, $n_{(2)}=\min\{m,n\}$.\\
$\log$ & Natural logarithm.\\
$I$, $\boldsymbol{0}$, $\boldsymbol{1}$ & The identity matrix, all-zero matrix, and all-one vector.\\
$e_i$ & Vector whose ith entry is 1 and others are 0s.\\
$M_{:j}$ & The $j$th column of matrix $M$.\\
$M_{ij}$ & The entry at the $i$th row and $j$th column of matrix $M$.\\
$M^T$ & Transpose of matrix $M$.\\
$M^\dag$ & Moore-Penrose pseudo-inverse of matrix $M$.\\
$|M|$ & $|M|_{ij}=|M_{ij}|$, $i=1,...,m, j=1,...,n$.\\
$||\cdot||_2$ & Euclidean norm for a vector, $||v||_2=\sqrt{\sum_{i}v_{i}^2}$.\\
$||\cdot||_*$ & Nuclear norm of a matrix (the sum of its singular values).\\
$||\cdot||_{\ell_0}$ & $\ell_0$ norm of a matrix (the number of non-zero entries).\\
$||\cdot||_{\ell_{2,0}}$ & $\ell_{2,0}$ norm of a matrix (the number of non-zero columns).\\
$||\cdot||_{\ell_1}$ & $\ell_1$ norm of a matrix, $||M||_{\ell_1}=\sum_{i,j}|M_{ij}|$. \\
$||\cdot||_{\ell_{2,1}}$ & $\ell_{2,1}$ norm of a matrix,
$||M||_{\ell_{2,1}}=\sum_{j}||M_{:j}||_2$.\\
$||\cdot||_F$ & Frobenius norm of a matrix, $||M||_F=\sqrt{\sum_{i,j}M_{ij}^2}$.\\
\hline
\end{tabular}
\end{center}
\end{table}

Since this paper involves multiple subspace recovery models, to
minimize confusion we name the models that minimize rank functions
and nuclear norms as the \emph{original} model and the
\emph{relaxed} model, respectively. We also name the models that
utilize the denoised data matrices for dictionaries as
\emph{robust} models, with a prefix ``R-".

\subsection{Robust Principal Component Analysis}\label{sec:RPCA}
Robust Principal Component Analysis
(R-PCA)~\citep{Wright:RPCA,Candes} is a robust version of PCA.
R-PCA aims at recovering a hidden low dimensional subspace from
the observed high dimensional data which have unknown sparse
corruptions. The low dimensional subspace and sparse corruptions
correspond to a low rank matrix $A_0$ and a sparse matrix $E_0$, respectively.
So the mathematical formulation of R-PCA is as follows:
\begin{equation}
\label{equ:original RPCA model}
\min_{A,E} \mbox{rank}(A)+\lambda||E||_{\ell_0},\ \ \mbox{s.t.}\ \ X=A+E,
\end{equation}
where $X=A_0+E_0\in \mathbb{R}^{m\times n}$ is the observation with data
samples being its columns and $\lambda>0$ is a regularization
parameter.

Since solving the original R-PCA is NP-hard, which prevents the
practical use of R-PCA, \cite{Candes} proposed solving its convex
surrogate, called Principal Component Pursuit or relaxed R-PCA
by our naming conventions, defined as follows:
\begin{equation}
\label{equ:relaxed RPCA model} \min_{A,E}
||A||_*+\lambda||E||_{\ell_1},\ \ \mbox{s.t.}\ \ X=A+E.
\end{equation}
This relaxation makes use of two facts. First, the nuclear norm is the convex envelope
of rank within the unit ball of matrix operation norm. Second,
the $\ell_1$ norm is the convex envelope of the $\ell_0$ norm.
\cite{Candes} further proved that when the rank of the structure
component $A_0$ is $O(n/\log^2 m)$ and $A_0$ is non-sparse (see incoherent conditions \eqref{equ: incoherence 1} and \eqref{equ: another incoherence}), and the number of
non-zeros of the noise matrix $E_0$ is $O(mn)$ (it is
remarkable that the magnitudes of noise could be arbitrarily
large), the solution of the convex relaxed R-PCA problem \eqref{equ:relaxed RPCA model} perfectly recovers the ground truth data matrix $A_0$ and noise matrix $E_0$ with an overwhelming probability. 

\comment{ Another remarkable benefit of PCP is that there is no
tuning parameter in the algorithm, as $\lambda=1/\sqrt{n}$ has
been demonstrated to be the best in theory, where $n$ is the
dimension of the matrix. }

\subsection{Low Rank Representation}
While R-PCA works well for a \emph{single} subspace with sparse
corruptions, it is unable to identify \emph{multiple} subspaces,
which is the main target of the subspace clustering problem. To
overcome this drawback, \cite{LiuG2,LiuG1} proposed Low Rank
Representation (LRR) modeled as follows:
\begin{equation}
\label{equ:noisy LRR} \min_{Z,E}
\mbox{rank}(Z)+\lambda||E||_{\ell_{2,0}}, \ \mbox{s.t.} \ X=XZ+E.
\end{equation}
The idea of LRR is to self-express the data, i.e., using data
itself as the dictionary, and then find the lowest-rank
representation matrix, supposing that the corruptions are sparse.
The pattern in the optimal $Z$, i.e., block diagonal structure, can help
identify the subspaces.

Again, due to the NP-hardness of the original LRR,
\cite{LiuG2,LiuG1} proposed solving the relaxed LRR instead:
\begin{equation}
\label{equ:relaxed noisy LRR} \min_{Z,E}
||Z||_*+\lambda||E||_{\ell_{2,1}}, \ \mbox{s.t.} \ X=XZ+E,
\end{equation}
where the $\ell_{2,1}$ norm is the convex envelope of the $\ell_{2,0}$
norm. They proved that if the fraction of corruptions does not
exceed a threshold, the row space of the ground truth $Z$ and the
indices of non-zero columns of the ground truth $E$ can be exactly
recovered~\citep{LiuG1}.

\subsection{Robust Low Rank Representation (Robust Shape Interaction and Low Rank Subspace Clustering)}\label{sec:RSI}
As mentioned above, LRR uses the data matrix itself as the
dictionary to represent data samples. This is not very reasonable
when the data contain severe noises or outliers. To remedy this
issue, \cite{Wei} suggested using denoised data as the dictionary
to express itself, resulting in the following model:
\begin{equation}
\label{equ:original RSI} \min_{Z,E}
\mbox{rank}(Z)+\lambda||E||_{\ell_{2,0}}, \ \mbox{s.t.} \
X-E=(X-E)Z.
\end{equation}
It is called the original Robust Shape Interaction (RSI) model.
Again, it has a relaxed version:
\begin{equation}
\label{equ:relaxed RSI} \min_{Z,E}
||Z||_*+\lambda||E||_{\ell_{2,1}}, \ \mbox{s.t.} \ X-E=(X-E)Z,
\end{equation}
by replacing rank and $\ell_{2,0}$ with their respective convex
envelopes.

Note that the relaxed RSI is still non-convex due to its
bilinear constraint, which may cause difficulty in finding its
globally optimal solution. \cite{Wei} first proved the following
result on the relaxed noiseless LRR, which is also the noiseless version of the relaxed RSI.
\begin{proposition}
\label{prop:solution to relaxed noiseless LRR}
The solution to relaxed noiseless LRR (RSI):
\begin{equation}
\label{equ:relaxed noiseless LRR} \min_Z \|Z\|_*, \ \ \mbox{s.t.}
\ \ A=AZ,
\end{equation}
is unique and given by $Z^*=V_AV_A^T$, where $U_A\Sigma_A V_A^T$
is the skinny SVD of $A$.
\end{proposition}
\begin{remark}
$V_A V_A^T$ can also be written as $A^\dag A$.
\eqref{equ:relaxed noiseless LRR} is a relaxed version of the
original noiseless LRR:
\begin{equation}
\label{equ:original noiseless LRR} \min_Z \mbox{rank}(Z), \ \
\mbox{s.t.} \ \ A=AZ.
\end{equation}
\end{remark}

$V_AV_A^T$ is called the Shape Interaction Matrix in the field of
structure from motion \citep{Costeira}. Hence model \eqref{equ:original RSI} is named Robust
Shape Interaction. $V_AV_A^T$
is block diagonal when the column vectors of $A$ lie strictly on
independent subspaces. The block diagonal pattern reveals the structure
of each subspace and therefore offers the possibility of subspace
clustering.


\cite{Wei} proposed to first solve the optimal $A^*$ and $E^*$ from
\begin{equation}
\label{equ: CSPCP} \min_{A,E} ||A||_*+\lambda||E||_{\ell_{2,1}},\
\ \mbox{s.t.}\ \ X=A+E,
\end{equation}
which we call the column sparse relaxed R-PCA since it is the convex relaxation of the original problem:
\begin{equation}
\label{equ: CSRPCA} \min_{A,E} \mbox{rank}(A)+\lambda||E||_{\ell_{2,0}},\
\ \mbox{s.t.}\ \ X=A+E.
\end{equation}
Then \cite{Wei} used $(Z^*,E^*)$ as the solution to the relaxed RSI
problem \eqref{equ:relaxed RSI}, where $Z^*$ is the Shape
Interaction Matrix of $A^*$ by Proposition \ref{prop:solution to relaxed noiseless LRR}. In this way, they deduced
the optimal solution. We will prove in Section~\ref{section:
proofs of the main results} that this is indeed true and actually
holds for arbitrary functions on $E$.

It is worth noting that~\cite{Xu:CSPCP} proved that problem
\eqref{equ: CSPCP} is capable of exactly recognizing the sparse
outliers and simultaneously recovering the column space of the
ground truth data under rather broad conditions. Our former work
\citep{Zhang2014NIPS} further showed that the parameter $\lambda=1/\sqrt{\log n}$
guarantees the success of the model, even when the rank of the intrinsic matrix
and the number of non-zero columns of the noise matrix are almost $O(n)$, where $n$ is the column number
of the input.

As a closely connected work of RSI, \cite{Favaro} and
\cite{Vidal:LRSC} proposed a similar model, called Low Rank
Subspace Clustering (LRSC):
\begin{equation}
\label{equ: LRSC} \min_{Z,A,E} ||Z||_*+\lambda||E||_{\ell_{1}}, \
\mbox{s.t.} \ X=A+E,A=AZ.
\end{equation}
One can see that LRSC only differs from RSI~\eqref{equ:relaxed
RSI} by the norm of $E$. The other difference is that LRSC adopts
the alternating direction method (ADM)~\citep{Lin2} to solve
\eqref{equ: LRSC}.

In order not to confuse the readers and to highlight that both RSI
and LRSC are robust versions of LRR, in the sequel we call them
Robust LRR (R-LRR) instead as our theoretical analysis allows for
arbitrary functions on $E$.

\comment{In this paper, unlike \cite{Wei} which has no strict
deduction, we show detailed demonstrations about the relationships
between problem \eqref{equ:original RPCA model},
\eqref{equ:original RSI} and \eqref{equ:relaxed RSI}. Our theory
focuses on the sparsity norm of the noise term $E$ as any
function, hence being much general. Based on our theory, more
complex problem \eqref{equ:original RSI} and \eqref{equ:relaxed
RSI} could be solved by studying relatively simple problem
\eqref{equ:original RPCA model}. Thus, we make up for the
theoretical blank about why the methods in \cite{Wei,Vidal:LRSC}
work well in practice. }

\subsection{Robust Latent Low Rank Representation}
Although LRR and R-LRR have been successful in applications such as
face recognition~\citep{LiuG2,LiuG1,Wei}, motion
segmentation~\citep{LiuG2,LiuG1,Favaro}, and image
classification~\citep{Bull:imageclassification,Zhang:dictionaryLRR},
they break down when the samples are insufficient, especially when
the number of samples is less than the dimensions of subspaces.
\cite{LiuG3} addressed this small sample problem by introducing
hidden data $X_H$ into the dictionary:
\begin{equation}
\label{equ: hidden LatLRR} \min_{R} \|R\|_*,\ \ \mbox{s.t.}\ \
X=[X,X_H]R.
\end{equation}
Obviously, it is impossible to solve problem \eqref{equ: hidden
LatLRR} because $X_H$ is unobserved. Nevertheless, by utilizing
Proposition~\ref{prop:solution to relaxed noiseless LRR},
\cite{LiuG3} proved that $X$ can be written as $X=XZ+LX$, where
both $Z$ and $L$ are low rank, resulting in the following Latent
Low Rank Representation (LatLRR) model:
\begin{equation}
\label{equ:noisy LatLRR model} \min_{Z,L,E}
\mbox{rank}(Z)+\mbox{rank}(L)+\lambda||E||_{\ell_0}, \ \mbox{s.t.}
\ X=XZ+LX+E,
\end{equation}
where sparse corruptions are considered. As a common practice, its
relaxed version is solved instead:
\begin{equation}
\label{equ:noisy nuclear_norm_LatLRR model} \min_{Z,L,E}
||Z||_*+||L||_*+\lambda||E||_{\ell_1}, \ \mbox{s.t.} \ X=XZ+LX+E.
\end{equation}

As in the case of LRR, when the data is very noisy or highly
corrupted, it is inappropriate to use $X$ itself as the
dictionary. So \cite{Zhang:RobustLatLRR} borrowed the idea of
R-LRR to use denoised data as the dictionary, giving rise to the
following Robust Latent LRR (R-LatLRR) model:
\begin{equation}
\label{equ:Robust LatLRR} \min_{Z,L,E}
\mbox{rank}(Z)+\mbox{rank}(L)+\lambda||E||_{\ell_0}, \ \mbox{s.t.}
\ X-E=(X-E)Z+L(X-E),
\end{equation}
and its relaxed version:
\begin{equation}
\label{equ:relaxed_Robust LatLRR} \min_{Z,L,E}
||Z||_*+||L||_*+\lambda||E||_{\ell_1}, \ \mbox{s.t.} \
X-E=(X-E)Z+L(X-E).
\end{equation}
Again the relaxed R-LatLRR model is non-convex. Surprisingly,
\cite{Zhang:Counterexample} proved that when there is no noise,
both the original R-LatLRR and relaxed R-LatLRR have
\emph{non-unique} closed-form solutions and they described the
complete solution sets. So like RSI, \cite{Zhang:Counterexample}
proposed applying R-PCA to separate $X$ into $X=A^*+E^*$. Next,
they found the sparsest solution among the solution set of
relaxed noiseless R-LatLRR:
\begin{equation}
\label{equ:noiseless nuclear norm LatLRR model} \min_{Z,L}
||Z||_*+||L||_*,\ \ \mbox{s.t.}\ \ A=AZ+LA,
\end{equation}
with $A$ being $A^*$. \eqref{equ:noiseless nuclear norm LatLRR
model} is a relaxed version of the original noiseless R-LatLRR
model:
\begin{equation}
\label{equ:noiseless rank LatLRR model} \min_{Z,L}
\mbox{rank}(Z)+\mbox{rank}(L),\ \ \mbox{s.t.}\ \ A=AZ+LA.
\end{equation}

In Section \ref{section: proofs of the main results}, we will
prove that the above two step procedure actually solves
\eqref{equ:relaxed_Robust LatLRR} correctly. More in-depth
analysis will also be provided.

\comment{Because of the discrete nature of the rank function,
doing as robust PCA, LRR uses the nuclear norm to approximate the
rank function, hence giving rise to the following formulation
instead:
\begin{equation}
\label{equ:relaxed noiseless LRR} \min_Z ||Z||_*, \ \ \mbox{s.t.}
\ \ A=AZ.
\end{equation}
}

\comment{Our former work on robust
LatLRR~\citep{Zhang:RobustLatLRR} suggested to use robust PCA
denoising the data first and then solve noiseless LatLRR
\eqref{equ:noiseless nuclear norm LatLRR model} with its closed
form solutions~\citep{Zhang:Counterexample}. In this paper, we
show that this procedure just follows the unified model
\eqref{equ:Robust LatLRR}. That is to say, model \eqref{equ:Robust
LatLRR} also has closed relationship with that of robust PCA
\eqref{equ:original RPCA model}.}

\subsection{Other Low Rank Models for Subspace Clustering}
\comment{\cite{Zhuang,LiuR2,Wang} also gave deformed version of
LRR. As a combination of sparsity based method~\citep{Elhamifar}
and LRR, \cite{Zhuang} proposed Non-Negative Low Rank and Sparse
Representation (NNLRSR) for semi-supervised learning tasks. It was
reported that NNLRSR could capture both the global mixture of
subspaces structure and the locally linear structure of the data,
thus being both discriminative and generative.}

In this subsection, we mention more low rank subspace recovery
models, although they are not our focus in this paper. Also aiming
at addressing the small sample issue, \cite{LiuR2} proposed Fixed
Rank Representation by requiring that the representation
matrix to be as close to a rank $r$ matrix as possible, where $r$
is a prescribed rank. Then the best rank $r$ matrix, which still
has the block diagonal structure, is used for subspace clustering.
\cite{Wang} extended LRR to address nonlinear multi-manifold
segmentation, where the error $E$ is regularized by the square of
Frobenius norm so that the kernel trick can be used. In \citep{Ni},
the authors augmented the LRR model with a semi-definiteness
constraint on the representation matrix $Z$. In contrast, the
representation matrices by R-LRR and R-LatLRR are both naturally
semi-definite as they are Shape Interaction Matrices.

\section{Main Results -- Relations among Low Rank Models}
\label{section: Theoretical Relationships between Low Rank Models}
In this section, we present the hidden connections among
representative low rank recovery models: R-PCA, R-LRR, and
R-LatLRR, although they appear different and have been proposed
for different purposes. Actually, our analysis holds for more
general models where the regularization on noise term $E$ can be
arbitrary. More specifically, the generalized models are:
\begin{equation}
\label{equ: generalized RPCA} \min_{A,E} \mbox{rank}(A)+\lambda
f(E),\ \ \mbox{s.t.}\ \ X=A+E,
\end{equation}
\begin{equation}
\label{equ: generalized PCP} \min_{A,E} \|A\|_*+\lambda f(E),\ \
\mbox{s.t.}\ \ X=A+E,
\end{equation}
\begin{equation}
\label{equ: generalized rank RSI} \min_{Z,E}
\mbox{rank}(Z)+\lambda f(E), \ \mbox{s.t.} \ X-E=(X-E)Z,
\end{equation}
\begin{equation}
\label{equ: generalized nuclear norm RSI} \min_{Z,E}
||Z||_*+\lambda f(E), \ \mbox{s.t.} \ X-E=(X-E)Z,
\end{equation}
\begin{equation}
\label{equ: generalized rank Robust LatLRR}
\min_{Z,L,E} \mbox{rank}(Z)+\mbox{rank}(L)+\lambda f(E),\ \ \mbox{s.t.}\ \ X-E=(X-E)Z+L(X-E),
\end{equation}
\begin{equation}
\label{equ: generalized nuclear norm Robust LatLRR} \min_{Z,L,E}
||Z||_*+||L||_*+\lambda f(E),\ \ \mbox{s.t.}\ \ X-E=(X-E)Z+L(X-E),
\end{equation}
where $f$ is any function. For brevity, we still call \eqref{equ:
generalized RPCA}-\eqref{equ: generalized nuclear norm Robust
LatLRR} the original R-PCA, relaxed R-PCA, original R-LRR,
relaxed R-LRR, original R-LatLRR, and relaxed R-LatLRR,
respectively, without mentioning ``generalized."

We show that the solutions to the above models are mutually
expressible, i.e., if we have a solution to one of the models, we
will obtain the solutions to other models in
\emph{closed-form} formulations. We will further show in
Section~\ref{section:applications of theoretical analysis} that
such mutual expressibility is useful.

It suffices to show that the solutions of the original R-PCA and
those of other models are mutually expressible, i.e., letting the
original R-PCA hinge all the above models. We summarize our
results as the following theorems.

\begin{theorem}[Connection between the original R-PCA and the original R-LRR]
\label{theorem: relationships about original RSI} For any
minimizer $(A^*,E^*)$ of the original R-PCA problem \eqref{equ:
generalized RPCA}, suppose $U_{A^*}\Sigma_{A^*}V_{A^*}^T$ is the
skinny SVD of the matrix $A^*$. Then $((A^*)^\dag
A^*+SV_{A^*}^T,E^*)$ is the optimal solution to the original R-LRR
problem \eqref{equ: generalized rank RSI}, where $S$ is any matrix
such that $V_{A^*}^TS=0$. Conversely, provided that $(Z^*,E^*)$ is
an optimal solution to the original R-LRR problem \eqref{equ:
generalized rank RSI}, $(X-E^*,E^*)$ is a minimizer of the
original R-PCA problem \eqref{equ: generalized RPCA}.
\end{theorem}

\begin{theorem}[Connection between the original R-PCA and the relaxed R-LRR]
\label{theorem: relationships about relaxed RSI} For any
minimizer $(A^*,E^*)$ of the original R-PCA problem \eqref{equ:
generalized RPCA}, the relaxed R-LRR problem \eqref{equ:
generalized nuclear norm RSI} has an optimal solution $((A^*)^\dag
A^*,E^*)$. Conversely, suppose that the relaxed R-LRR problem
\eqref{equ: generalized nuclear norm RSI} has a minimizer
$(Z^*,E^*)$, then $(X-E^*,E^*)$ is an optimal solution to the
original R-PCA problem \eqref{equ: generalized RPCA}.
\end{theorem}

\begin{remark}
According to Theorem \ref{theorem: relationships about relaxed
RSI}, the relaxed R-LRR can be viewed as denoising the data
first by the original R-PCA and then adopting the shape
interaction matrix of the denoised data matrix as the affinity
matrix. Such a procedure is exactly the same as that in \citep{Wei}
which was proposed out of heuristics and for which there was no
proof provided.
\end{remark}

\begin{theorem}[Connection between the original R-PCA and the original R-LatLRR]
\label{theorem: relationships about original Robust LatLRR} Let
the pair $(A^*,E^*)$ be any optimal solution to the original R-PCA
problem \eqref{equ: generalized RPCA}. Then the original R-LatLRR
model \eqref{equ: generalized rank Robust LatLRR} has minimizers
$(Z^*,L^*,E^*)$, where
\begin{equation}
Z^*=V_{A^*}\widetilde{W}V_{A^*}^T+S_1\widetilde{W}V_{A^*}^T,\ \ L^*=U_{A^*}\Sigma_{A^*}(I-\widetilde{W})\Sigma_{A^*}^{-1}U_{A^*}^T+U_{A^*}\Sigma_{A^*}(I-\widetilde{W})S_2,
\end{equation}
$\widetilde{W}$ is any idempotent matrix and $S_1$ and $S_2$ are
any matrices satisfying:
\begin{enumerate}
\item $V_A^TS_1=0$ and $S_2U_A=0$; and \item
rank$(S_1)\leq$rank$(\widetilde{W})$ and
rank$(S_2)\leq$rank$(I-\widetilde{W})$.
\end{enumerate}
Conversely, let $(Z^*,L^*,E^*)$ be any optimal solution to the
original R-LatLRR \eqref{equ: generalized rank Robust LatLRR}.
Then $(X-E^*,E^*)$ is a minimizers of the original R-PCA problem
\eqref{equ: generalized RPCA}.
\end{theorem}

\begin{theorem}[Connection between the original R-PCA and the relaxed R-LatLRR]
\label{theorem: relationships about relaxed Robust LatLRR} Let
the pair $(A^*,E^*)$ be any optimal solution to the original R-PCA
problem \eqref{equ: generalized RPCA}. Then the relaxed R-LatLRR
model \eqref{equ: generalized nuclear norm Robust LatLRR} has
minimizers $(Z^*,L^*,E^*)$, where
\begin{equation}
\label{equ: solutions to noiseless nuclear norm LatLRR}
Z^*=V_{A^*}\widehat{W}V_{A^*}^T,\ \ L^*=U_{A^*}(I-\widehat{W})U_{A^*}^T,
\end{equation}
and $\widehat{W}$ is any block diagonal matrix satisfying:
\begin{enumerate}
\item its blocks are compatible with $\Sigma_{A^*}$, i.e., if
$[\Sigma_{A^*}]_{ii}\not=[\Sigma_{A^*}]_{jj}$ then
$[\widehat{W}]_{ij}$=0; and \item both $\widehat{W}$ and
$I-\widehat{W}$ are positive semi-definite.
\end{enumerate}
Conversely, let $(Z^*,L^*,E^*)$ be any optimal solution to the
relaxed R-LatLRR \eqref{equ: generalized nuclear norm Robust
LatLRR}. Then $(X-E^*,E^*)$ is a minimizer of the original R-PCA
problem \eqref{equ: generalized RPCA}.
\end{theorem}
Figure \ref{fig: Relationships} illustrates our theorems by
putting the original R-PCA at the center of the low rank subspace
clustering models under consideration.

\begin{figure}
\centering
\includegraphics[width=0.8\textwidth]{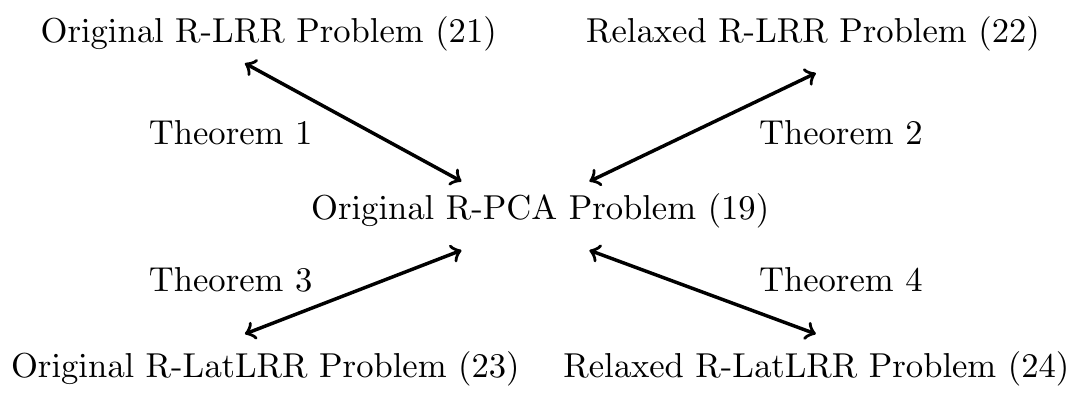}
\caption{Visualization of the relationship among problem
\eqref{equ: generalized rank RSI}, \eqref{equ: generalized nuclear
norm RSI}, \eqref{equ: generalized rank Robust LatLRR},
\eqref{equ: generalized nuclear norm Robust LatLRR}, and
\eqref{equ: generalized RPCA}, where an arrow means that a
solution to one problem could be used to express a solution (or
solutions) to the other problem in a closed form.} \label{fig:
Relationships}
\end{figure}

\comment{To better establish the connections, the following
corollary shows that \eqref{equ: generalized rank RSI},
\eqref{equ: generalized nuclear norm RSI}, \eqref{equ: generalized
rank Robust LatLRR} and \eqref{equ: generalized nuclear norm
Robust LatLRR} are also inter-expressed, where robust PCA serves
as a middle procedure for bridging the connections between various
models. }

By the above theorems, we can easily have the following corollary.
\begin{corollary}
\label{corollary: complete connections} The solutions to the
original R-PCA \eqref{equ: generalized RPCA}, original R-LRR
\eqref{equ: generalized rank RSI}, relaxed R-LRR \eqref{equ:
generalized nuclear norm RSI}, original R-LatLRR \eqref{equ:
generalized rank Robust LatLRR}, and relaxed R-LatLRR
\eqref{equ: generalized nuclear norm Robust LatLRR} are all
mutually expressible.
\end{corollary}

\begin{remark}
\label{remark: better solution}
According to the above results, once we obtain a globally optimal solution to the original R-PCA \eqref{equ: generalized RPCA}, we can obtain globally optimal solutions to the original and relaxed R-LRR and R-LatLRR problems. Although in general solving the original R-PCA \eqref{equ: generalized RPCA} is NP hard, under certain condition (see Section \ref{subsection: Better Solution for Subspace Recovery}) its globally optimal solution can be obtained with an overwhelming probability by solving the relaxed R-PCA \eqref{equ: generalized PCP}. If one solves the original and relaxed R-LRR or R-LatLRR directly, e.g., by ADM, there is no analysis on whether their globally optimal solutions can be attained, due to their non-convex nature. In this sense, we say that we can obtain a better solution for the original and relaxed R-LRR and R-LatLRR if we reduce them to the original R-PCA. Our numerical experiments in Section \ref{subsection:comparison of optimality on synthetic data} testify to our claims.
\end{remark}

\section{Proofs of Main Results}
\label{section: proofs of the main results} In this section, we
provide detailed proofs of the four theorems in the previous
section.

\subsection{Connection between R-PCA and R-LRR}
The following lemma is useful throughout the proof of Theorem
\ref{theorem: relationships about original RSI}.

\begin{lemma}[\cite{Zhang:Counterexample}]
\label{lemma: solution to original LRR} Suppose $U_A\Sigma_AV_A^T$
is the skinny SVD of $A$. Then the complete solutions to
\eqref{equ:original noiseless LRR} are $Z^*=A^\dag A+SV_A^T$,
where $S$ is any matrix such that $V_A^TS=0$.
\end{lemma}
Using Lemma \ref{lemma: solution to original LRR}, we can prove
Theorem \ref{theorem: relationships about original RSI}.

\begin{proof}\textbf{(Theorem \ref{theorem: relationships about original RSI})}
We first prove the first part of the theorem. Since $(A^*,E^*)$ is
a feasible solution to problem \eqref{equ: generalized RPCA}, it
is easy to check that $((A^*)^\dag A^*+SV_{A^*}^T,E^*)$ is also
feasible to \eqref{equ: generalized rank RSI} by using a
fundamental property of Moore-Penrose pseudo-inverse: $YY^\dag
Y=Y$. Now suppose that $((A^*)^\dag A^*+SV_{A^*}^T,E^*)$ is not an
optimal solution to \eqref{equ: generalized rank RSI}. Then there
exists an optimal solution to \eqref{equ: generalized rank RSI}, denoted by
$(\widetilde{Z},\widetilde{E})$, such that
\begin{equation}
\label{equ: 1.1}
\mbox{rank}(\widetilde{Z})+\lambda
f(\widetilde{E})<\mbox{rank}((A^*)^\dag A^*+SV_{A^*}^T)+\lambda
f(E^*)
\end{equation}
and meanwhile $(\widetilde{Z},\widetilde{E})$ is feasible: $X-\widetilde{E}=(X-\widetilde{E})\widetilde{Z}$.
Since $(\widetilde{Z},\widetilde{E})$ is optimal to problem \eqref{equ: generalized rank RSI}, by Lemma \ref{lemma: solution to original LRR}, we fix $\widetilde{E}$ and have
\begin{equation}
\label{equ: 1.2}
\begin{split}
\mbox{rank}(\widetilde{Z})+\lambda f(\widetilde{E})&=\mbox{rank}((X-\widetilde{E})^\dag(X-\widetilde{E}))+\lambda f(\widetilde{E})\\
&=\mbox{rank}(X-\widetilde{E})+\lambda f(\widetilde{E}).
\end{split}
\end{equation}
On the other hand,
\begin{equation}
\label{equ: 1.3}
\mbox{rank}((A^*)^\dag A^*+SV_{A^*}^T)+\lambda f(E^*)=\mbox{rank}(A^*)+\lambda f(E^*).
\end{equation}
From \eqref{equ: 1.1}, \eqref{equ: 1.2}, and \eqref{equ: 1.3}, we have
\begin{equation}
\mbox{rank}(X-\widetilde{E})+\lambda f(\widetilde{E})<\mbox{rank}(A^*)+\lambda f(E^*),
\end{equation}
which leads to a contradiction with the optimality
of $(A^*,E^*)$ to R-PCA~\eqref{equ: generalized RPCA}.

We then prove the converse, also by contradiction. Suppose that
$(Z^*,E^*)$ is a minimizer to the original R-LRR problem
\eqref{equ: generalized rank RSI}, while $(X-E^*,E^*)$ is not a
minimizer to the R-PCA problem \eqref{equ: generalized RPCA}. Then
there will be a better solution to problem \eqref{equ: generalized RPCA}, termed
$(\widetilde{A},\widetilde{E})$, which satisfies
\begin{equation}
\label{equ: RSI and RPCA middle}
\mbox{rank}(\widetilde{A})+\lambda
f(\widetilde{E})<\mbox{rank}(X-E^*)+\lambda f(E^*).
\end{equation}
Fixing $E$ as $E^*$ in
\eqref{equ: generalized rank RSI}, by
Lemma \ref{lemma: solution to original LRR} and the
optimality of $Z^*$, we infer that
\begin{equation}
\label{equ: RSI and RPCA right}
\begin{split}
\mbox{rank}(X-E^*)+\lambda f(E^*)&=\mbox{rank}((X-E^*)^\dag(X-E^*))+\lambda f(E^*)\\
&=\mbox{rank}(Z^*)+\lambda f(E^*).
\end{split}
\end{equation}
On the other hand,
\begin{equation}
\label{equ: RSI and RPCA left}
\mbox{rank}(\widetilde{A})+\lambda f(\widetilde{E})=\mbox{rank}(\widetilde{A}^\dag\widetilde{A})+\lambda f(\widetilde{E}),
\end{equation}
where we have utilized another property of the Moore-Penrose
pseudo-inverse: rank($Y^\dag Y$)=rank(Y). Combining \eqref{equ:
RSI and RPCA middle}, \eqref{equ: RSI and RPCA right}, and
\eqref{equ: RSI and RPCA left}, we have
\begin{equation}\label{equ: RSI and RPCA final}
\mbox{rank}(\widetilde{A}^\dag\widetilde{A})+\lambda
f(\widetilde{E})<\mbox{rank}(Z^*)+\lambda f(E^*).
\end{equation}
Notice that $(\widetilde{A}^\dag\widetilde{A},\widetilde{E})$ satisfies
the constraint of the original R-LRR problem \eqref{equ:
generalized rank RSI} due to $\widetilde{A}+\widetilde{E}=X$ and
$\widetilde{A}(\widetilde{A}^\dag\widetilde{A})=\widetilde{A}$.
The inequality \eqref{equ: RSI and RPCA final} leads to a contradiction with the
optimality of the pair $(Z^*,E^*)$ for R-LRR.

Thus we finish the proof.
\end{proof}

Now we prove Theorem \ref{theorem: relationships about relaxed
RSI}. Proposition \ref{prop:solution to relaxed noiseless LRR} is critical for
the proof.

\comment{
\begin{lemma}
\label{lemma:nuclear norm and rank} For any matrix $Y$, $||Y^\dag
Y||_*=\mbox{rank}(Y)$.
\end{lemma}
\begin{proof}
This can be easily verified by the skinny SVD of $Y$.
\end{proof}
}

\comment{
\begin{lemma}[\cite{LiuG1}]
\label{lemma: solution to relaxed LRR}
The minimizer to the relaxed LRR problem \eqref{equ:relaxed noiseless LRR} is unique and is equal to $Z^*=A^\dag A$.
\end{lemma}
}

\begin{proof}\textbf{(Theorem \ref{theorem: relationships about relaxed RSI})}
The proof is similar to that of Theorem \ref{theorem: relationships about original RSI}. The only difference is that we need to use Proposition \ref{prop:solution to relaxed noiseless LRR} rather than Lemma \ref{lemma: solution to original LRR}.
\end{proof}

\subsection{Connection between R-PCA and R-LatLRR}
\comment{Paralleled with above discussion, we also build the
relationships between robust PCA problem \eqref{equ: generalized
RPCA} and R-LatLRR\footnote{R-LatLRR~\citep{Zhang:RobustLatLRR} is
a subspace clustering method proceeding by denoising the data
first via robust PCA and then solving LatLRR \eqref{equ:noiseless
nuclear norm LatLRR model} with its closed form solution. In this
paper, without confusion, we use the term robust LatLRR to
represent models \eqref{equ: generalized rank Robust LatLRR} and
\eqref{equ: generalized nuclear norm Robust LatLRR}. Just as our
following theorems show, the method proposed in our former
work~\citep{Zhang:RobustLatLRR} coincides with such two models.}
problem.}

Now we prove the mutual expressibility between the solutions of
R-PCA and R-LatLRR. Our former work~\citep{Zhang:Counterexample}
gives the complete closed-form solutions to noiseless R-LatLRR
problems \eqref{equ:noiseless rank LatLRR model} and
\eqref{equ:noiseless nuclear norm LatLRR model}, which are both
critical to our proofs.
\begin{lemma}[\cite{Zhang:Counterexample}]
\label{lemma: solution to original LatLRR} Suppose
$U_A\Sigma_AV_A^T$ is the skinny SVD of a denoised data matrix
$A$. Then the complete solutions to the original noiseless R-LatLRR problem
\eqref{equ:noiseless rank LatLRR model}
are as follows:
\begin{equation}
\label{equ: solutions to noiseless rank LatLRR}
Z^*=V_A\widetilde{W}V_A^T+S_1\widetilde{W}V_A^T\ \mbox{and}\ L^*=U_A\Sigma_A(I-\widetilde{W})\Sigma_A^{-1}U_A^T+U_A\Sigma_A(I-\widetilde{W})S_2,
\end{equation}
where $\widetilde{W}$ is any idempotent matrix and $S_1$ and $S_2$
are any matrices satisfying:
\begin{enumerate}
\item $V_A^TS_1=0$ and $S_2U_A=0$; and \item
rank$(S_1)\leq$rank$(\widetilde{W})$ and
rank$(S_2)\leq$rank$(I-\widetilde{W})$.
\end{enumerate}
\end{lemma}

Now we are ready to prove Theorem
\ref{theorem: relationships about original Robust LatLRR}.

\begin{proof}\textbf{(Theorem \ref{theorem: relationships about original Robust LatLRR})}
The proof is similar to that of Theorem \ref{theorem: relationships about original RSI}. The only difference is that we need to use Lemma \ref{lemma: solution to original LatLRR} rather than Lemma \ref{lemma: solution to original LRR}.
\end{proof}

The following lemma is helpful for proving the connection between the
R-PCA \eqref{equ: generalized RPCA} and the relaxed R-LatLRR
\eqref{equ: generalized nuclear norm Robust LatLRR}.

\begin{lemma}[\cite{Zhang:Counterexample}]
\label{lemma: solution to relaxed LatLRR} Suppose
$U_A\Sigma_AV_A^T$ is the skinny SVD of a denoised data matrix
$A$. Then the complete optimal
solutions to the relaxed noiseless R-LatLRR problem
\eqref{equ:noiseless nuclear norm LatLRR model} are as follows:
\begin{equation}
Z^*=V_A\widehat{W}V_A^T\ \mbox{and}\ L^*=U_A(I-\widehat{W})U_A^T,
\end{equation}
where $\widehat{W}$ is any block diagonal matrix satisfying:
\begin{enumerate}
\item its blocks are compatible with $\Sigma_A$, i.e., if
$[\Sigma_A]_{ii}\not=[\Sigma_A]_{jj}$ then $[\widehat{W}]_{ij}$=0;
and \item both $\widehat{W}$ and $I-\widehat{W}$ are positive
semi-definite.
\end{enumerate}
\end{lemma}

Now we are ready to prove Theorem \ref{theorem: relationships about relaxed Robust LatLRR}.

\begin{proof}\textbf{(Theorem \ref{theorem: relationships about relaxed Robust LatLRR})}
The proof is similar to that of Theorem \ref{theorem: relationships about original RSI}. The only difference is that we need to use Lemma \ref{lemma: solution to relaxed LatLRR} rather than Lemma \ref{lemma: solution to original LRR}.
\comment{
The proof is similar to that of Theorem \ref{theorem:
relationships about original Robust LatLRR}. For brevity, we only
sketch it. Suppose $(\widetilde{Z},\widetilde{L},\widetilde{E})$
is a better solution than $(Z^*,L^*,E^*)$ to relaxed R-LatLRR
\eqref{equ: generalized nuclear norm Robust LatLRR}, i.e.,
\begin{equation}
\label{equ: relaxed Robust LatLRR middle}
||\widetilde{Z}||_*+||\widetilde{L}||_*+\lambda f(\widetilde{E})<||Z^*||_*+||L^*||_*+\lambda f(E^*).
\end{equation}
Then according to Lemma \ref{lemma: solution to relaxed LatLRR},
we have
\begin{equation}
\label{equ: relaxed Robust LatLRR left}
||\widetilde{Z}||_*+||\widetilde{L}||_*+\lambda f(\widetilde{E})=\mbox{rank}(X-\widetilde{E})+\lambda f(\widetilde{E}),
\end{equation}
\begin{equation}
\label{equ: relaxed Robust LatLRR right}
||Z^*||_*+||L^*||_*+\lambda f(E^*)=\mbox{rank}(X-E^*)+\lambda f(E^*).
\end{equation}
Thus we get a contradiction by considering \eqref{equ: relaxed
Robust LatLRR middle}, \eqref{equ: relaxed Robust LatLRR left},
and \eqref{equ: relaxed Robust LatLRR right}.

Conversely, suppose the R-PCA problem \eqref{equ: generalized
RPCA} has a better solution $(\widetilde{A},\widetilde{E})$ than
$(X-E^*,E^*)$. Similar to Theorem \ref{theorem: relationships
about original Robust LatLRR}, the following inequality
\begin{equation}
||\widetilde{A}^\dag\widetilde{A}||_*+||\boldsymbol{0}||_*+\lambda
f(\widetilde{E})<||Z^*||_*+||L^*||_*+\lambda f(E^*)
\end{equation}
can be deduced and a contradiction results.
}
\end{proof}

Finally, viewing R-PCA as a hinge we connect all the models
considered in Section~\ref{section: Theoretical Relationships
between Low Rank Models}. We now prove Corollary \ref{corollary:
complete connections}.
\begin{proof}\textbf{(Corollary \ref{corollary: complete connections})}
According to Theorems \ref{theorem: relationships about original
RSI}, \ref{theorem: relationships about relaxed RSI},
\ref{theorem: relationships about original Robust LatLRR}, and
\ref{theorem: relationships about relaxed Robust LatLRR}, the
solution to R-PCA and those of other models are mutually
expressible. Next, we build the relationships among \eqref{equ:
generalized rank RSI}, \eqref{equ: generalized nuclear norm RSI},
\eqref{equ: generalized rank Robust LatLRR}, and \eqref{equ:
generalized nuclear norm Robust LatLRR}. For simplicity, we only
take \eqref{equ: generalized rank RSI} and \eqref{equ: generalized
nuclear norm RSI} for example. The proofs of the remaining
connections are similar.

Suppose $(Z^*,E^*)$ is optimal to the original R-LRR problem
\eqref{equ: generalized rank RSI}. Then based on Theorem
\ref{theorem: relationships about original RSI}, $(X-E^*,E^*)$ is
an optimal solution to the R-PCA problem \eqref{equ: generalized
RPCA}. Then Theorem \ref{theorem: relationships about relaxed
RSI} concludes that $((X-E^*)^\dag(X-E^*),E^*)$ is a minimizer of
the relaxed R-LRR problem \eqref{equ: generalized nuclear norm
RSI}. Conversely, suppose that $(Z^*,E^*)$ is optimal to the
relaxed R-LRR problem \eqref{equ: generalized nuclear norm RSI}.
By Theorems \ref{theorem: relationships about original RSI} and
\ref{theorem: relationships about relaxed RSI}, we conclude that
$((X-E^*)^\dag(X-E^*)+SV_{X-E^*}^T,E^*)$ is an optimal solution to
the original R-LRR problem \eqref{equ: generalized rank RSI},
where $V_{X-E^*}$ is the matrix of right singular vectors in the
skinny SVD of $X-E^*$ and $S$ is any matrix satisfying
$V_{X-E^*}^TS=0$.
\end{proof}

\section{Applications of the Theoretical Analysis}
\label{section:applications of theoretical analysis} In this
section, we discuss pragmatic values of our theoretical results in
Section~\ref{section: Theoretical Relationships between Low Rank
Models}. As one can see in Figure~\ref{fig: Relationships}, we put
R-PCA at the center of all the low rank models under consideration
because it is the simplest one among all, which implies that we prefer
deriving solutions of other models from that of R-PCA. For
simplicity, we call our two step approach, i.e., first reducing to
R-PCA and then expressing desired solution by the solution of R-PCA, as REDU-EXPR
method. There are two advantages of REDU-EXPR.
\begin{itemize}
\item We could obtain {\it better} solutions to other low rank models (cf. Remark \ref{remark: better solution}). R-PCA has a solid theoretical foundation.
\cite{Candes} proved that under certain conditions solving the
relaxed R-PCA~\eqref{equ:relaxed RPCA model}, which is convex,
can recover the ground truth solution at an overwhelming
probability (See Section~\ref{subsubsection: Sparse Element-wise Noises}). \cite{Xu:CSPCP,Zhang2014NIPS} also
proved similar results for column sparse relaxed
R-PCA~\eqref{equ: CSPCP} (See Section~\ref{subsubsection: Sparse Column-wise Noises}). Then by the mutual-expressibility of
solutions, we could also obtain globally optimal solutions to
other models. In contrast, the optimality of a solution is
uncertain if we solve other models using specific algorithms,
e.g., ADM~\citep{Lin2}, due to their non-convex nature. \item We
could have \emph{much faster} algorithms for other low rank
models. Due to the simplicity of R-PCA, solving R-PCA is much
faster than other models. In particular, the expensive $O(mn^2)$
complexity of matrix-matrix multiplication (between $X$ and $Z$ or
$L$) could be avoided. Moreover, there are low complexity
randomized algorithms for solving R-PCA, making the computational
cost of solving other models even lower. In particular, we propose
an $\ell_{2,1}$ filtering algorithm for column sparse relaxed R-PCA
(\eqref{equ: generalized PCP} with $f(E)=\|E\|_{\ell_{2,1}}$). If
one is directly faced with other models, it is non-trivial to
design low complexity algorithms (either deterministic or
randomized\footnote{We have to emphasize that although there is
linear time SVD algorithm~\citep{avron2010blendenpik,mahoney2011randomized} for computing SVD at low cost,
which is typically needed in the existing solvers for all models,
linear time SVD is known to have relative error. Moreover, even adopting
linear time SVD the whole complexity could still be $O(mn^2)$ \emph{due
to matrix-matrix multiplications \textbf{outside} the SVD computation in each iteration} if there is no
careful treatment.}).
\end{itemize}
In summary, based on our analysis we could achieve low rankness
based subspace clustering with better performance and faster
speed.

\subsection{Better Solution for Subspace Recovery}
\label{subsection: Better Solution for Subspace Recovery}

As stated above, reducing to R-PCA could help overcome the
non-convexity issue of the low rank recovery models we consider.
We defer the numerical verification of this claim until
Section~\ref{subsection:comparison of optimality on synthetic
data}. In this subsection, we discuss the theoretical conditions under which
reducing to R-PCA succeeds for subspace clustering problem.

We focus on the application of Theorem \ref{theorem: relationships about relaxed RSI}, which shows that given
the solution $(A^*,E^*)$ to R-PCA problem \eqref{equ: generalized RPCA}, the optimal solution to relaxed
R-LRR problem \eqref{equ: generalized nuclear norm RSI} is presented by $((A^*)^\dag A^*, E^*)$. Note that
$(A^*)^\dag A^*$ is called the Shape Interaction Matrix in the field of
structure from motion, and has been proven
to be block diagonal by \citep{Costeira} when the column vectors of $A^*$ lie strictly on
independent subspaces and the sampling number of $A^*$ from each subspace is larger than the subspace dimension~\citep{LiuG1}. The block diagonal pattern reveals the structure
of each subspace and hence offers the possibility of subspace
clustering. Thus to illustrate the success of our approach, the remainder is to show that under which conditions
R-PCA problem exactly recovers the noiseless data matrix, or correctly recognizes the indices of noises.
We discuss the cases where the corruptions are sparse element-wise noises, sparse column-wise noises, and dense Gaussian noises, respectively.

\subsubsection{Sparse Element-wise Noises}
\label{subsubsection: Sparse Element-wise Noises}
Suppose each column of the data matrix is an observation. In the case where the corruptions are sparse element-wise noises, we assume that the positions of the corrupted elements sparsely and uniformly distribute on the input matrix. In this case, we consider the usage of $\ell_1$ norm, i.e., model \eqref{equ:relaxed RPCA model} to remove the corruptions.

\cite{Candes} gave certain conditions under which model \eqref{equ:relaxed RPCA model} exactly recovers the noiseless data $A_0$ from the corrupted observations $X=A_0+E_0\in\mathbb{R}^{m\times n}$. We apply them to the success conditions of our approach. Firstly, to avoid the possibility that the low rank part $A_0$ is sparse, $A_0$ needs to satisfy the following incoherent conditions:
\begin{subequations}
\begin{align}
&\ \ \ \ \ \ \ \ \ \ \ \ \ \max_i ||V_0^Te_i||_2\le\sqrt{\frac{\mu r}{n}},\label{equ: incoherence 1}\\
&\max_i ||U_0^Te_i||_2\le\sqrt{\frac{\mu r}{m}},\ \ ||U_0V_0^T||_\infty\le\sqrt\frac{\mu r}{mn},\label{equ: another incoherence}
\end{align}
\end{subequations}
where $U_0\Sigma_0V_0^T$ is the skinny SVD of $A_0$, $r=\mbox{rank}(A_0)$, and $\mu$ is a constant.
\comment{
Notice that for the subspace clustering problem where the columns of $A_0$ are drawn from multiple independent subspaces, the above incoherence is no harder to satisfy than the case where data are drawn from a single subspace. This is because we can view the union of multiple subspaces as a single subspace with higher dimension.
}
The second assumption for the success of the algorithm is that the dimension of the sum of the subspaces is sufficiently low and the support number $s$ of the noise matrix $E_0$ is not too large. Namely,
\begin{equation}
\label{equ: element-wise range}
\mbox{rank}(A_0)\le\rho_r\frac{n_{(2)}}{\mu(\log n_{(1)})^2}\ \ \ \ \mbox{and}\ \ \ \ s\le\rho_s mn,
\end{equation}
where $\rho_r$ and $\rho_s$ are numerical constants. Under these conditions, \cite{Candes} justified that relaxed R-PCA \eqref{equ:relaxed RPCA model} with $\lambda=1/\sqrt{n_{(1)}}$ exactly recovers the noiseless data $A_0$. Thus the algorithm of reducing to R-PCA succeeds, as long as the subspaces are independent and the sampling number from each subspace is larger than the subspace dimension~\citep{LiuG1}.

\subsubsection{Sparse Column-wise Noises}
\label{subsubsection: Sparse Column-wise Noises}
In more general case, the noises exist in small number of columns, i.e., each non-zero column of $E_0$ corresponds to a corruption. In this case, we consider the usage of $\ell_{2,1}$ norm, i.e., model \eqref{equ: CSPCP} to remove the corruptions.

There have been several literatures that investigated the theoretical conditions under which column sparse relaxed R-PCA \eqref{equ: CSPCP} succeeds~\citep{Xu:CSPCP,ICML2011Chen_469,Zhang2014NIPS}. To the best of our knowledge, our discovery in \citep{Zhang2014NIPS} gave the broadest range under which model \eqref{equ: CSPCP} exactly identifies the indices of noises. Notice that it is impossible to recover a corrupted sample into its right subspace, since the magnitude of noises here can be arbitrarily large. Moreover, for the observations like
\begin{equation}
M=
\begin{bmatrix}
0 & 1 & 1 & \cdots & 1 & 0 & 0 & \cdots & 0\\
0 & 0 & 0 & \cdots & 0 & 1 & 1 & \cdots & 1\\
1 & 0 & 0 & \cdots & 0 & 0 & 0 & \cdots & 0\\
\vdots & \vdots & \vdots &  & \vdots & \vdots & \vdots &  & \vdots\\
0 & 0 & 0 & \cdots & 0 & 0 & 0 & \cdots & 0\\
\end{bmatrix},
\end{equation}
where the first column is the corrupted sample while others are noiseless, it is even harder to identify that the ground truth of the first column of $M$ belongs to the space Range$(e_1)$ or the space Range$(e_2)$.
So we remove the corrupted observation identified by the algorithm rather than exactly recovering its ground truth, and use the remaining noiseless data to reveal the real subspaces structure.

According to our discovery~\citep{Zhang2014NIPS}, the success of model \eqref{equ: CSPCP} requires the incoherence as well. However, only the condition \eqref{equ: incoherence 1} is needed, which is sufficient to guarantee that the low rank part cannot be column sparse. Similarly, to avoid the column sparse part being low rank when the number of its non-zero columns is comparable to $n$, we assume $||\mathcal{B}(E_0)||\le\sqrt{\log n}/4$, where $\mathcal{B}(E_0)=\{H:
\mathcal{P}_{\mathcal{I}^\perp}(H)=\mathbf{0};
H_{:j}=[E_0]_{:j}/||[E_0]_{:j}||_2,
[E_0]_{:j}\in\mathcal{I}$\}, $\mathcal{I}=\{j: [E_0]_{:j}\not=\mathbf{0}\}$, and $\mathcal{P}_{\mathcal{I}^\perp}$ is a projection onto the complement of $\mathcal{I}$. The dimension of the sum of the subspaces also requires to be low and the column support number $s$ of the noise matrix $E_0$ is not too large. More specifically,
\begin{equation}
\label{equ: column-wise range}
\mbox{rank}(A_0)\le\rho_r\frac{n_{(2)}}{\mu\log n_{(1)}}\ \ \ \ \mbox{and}\ \ \ \ s\le\rho_s n,
\end{equation}
where $\rho_r$ and $\rho_s$ are numerical constants. Note that the range of the successful rank$(A_0)$ in \eqref{equ: column-wise range} is broader than that of \eqref{equ: element-wise range}, and has been proved to be tight~\citep{Zhang2014NIPS}. Moreover, to avoid $[A_0+E_0]_{:j}$ lying in an incorrect subspace, we assume $[E_0]_{:j}\not\in\mbox{Range}(A_0)$ for $\forall j\in \mathcal{I}$. Under these conditions, our theorem justifies that column sparse relaxed R-PCA \eqref{equ: CSPCP} with $\lambda=1/\sqrt{\log n_{(1)}}$ exactly recognizes the indices of noises. Thus our approach succeeds.

\comment{
Suppose the data are drawn from 5 independent subspaces of dimension $D$ in the ambient space $\mathbb{R}^{5D}$. $100D$ samples are drawn from each subspace and so we have a $5D\times 100D$ data matrix.}

\subsubsection{Dense Gaussian Noises}
\label{subsubsection: Dense Gaussian Noises}
Assume that the data $A_0$ lies in an $r$-dimension subspace where $r$ is relatively small. For dense Gaussian noises, we consider the usage of squared Frobenius norm, leading to the following relaxed R-LRR problem:
\begin{equation}
\label{equ:Favaro conclusion}
\min_{A,Z} ||Z||_*+\lambda||E||_F^2,\ \ \mbox{s.t.}\ \ A=AZ,\ \ X=A+E.
\end{equation}
We quote the following result from
\citep{Favaro}, which gave the closed-form solution to the problem \eqref{equ:Favaro conclusion}. Based on
our results in Section \ref{section: Theoretical Relationships between Low Rank Models}, we give a new proof.

\begin{corollary}[\cite{Favaro}]\label{coro:favaro}
Let $X=U\Sigma V^T$ be the SVD of the data matrix $X$. Then the
optimal solution to \eqref{equ:Favaro conclusion}
is given by $A^*=U_1\Sigma_1V_1^T$ and $Z^*=V_1V_1^T$, where
$\Sigma_1$, $U_1$, and $V_1$ correspond to the top $r=\arg\min_k
k+\lambda\sum_{i>k}\sigma_i^2$ singular values and singular
vectors of $X$, respectively.
\end{corollary}
\begin{proof} The optimal solution to problem
\begin{equation}
\min_A \mbox{rank}(A)+\lambda||X-A||_F^2,
\end{equation}
is $A^*=U_1\Sigma_1V_1^T$, where $\Sigma_1$, $U_1$, and $V_1$
correspond to the top $r=\arg\min_k k+\lambda\sum_{i>k}\sigma_i^2$
singular values and singular vectors of $X$, respectively. This
can be easily seen by probing different rank $k$ of $A$ and
observing that $\min\limits_{\mbox{\small rank}(A)\leq
k}||X-A||_F^2 = \sum_{i>k}\sigma_i^2$.

Next, according to Theorem \ref{theorem: relationships about
relaxed RSI}, where $f$ is chosen as the squared Frobenius
norm, the optimal solution to problem \eqref{equ:Favaro
conclusion} is given by $A^*=U_1\Sigma_1V_1^T$ and $Z^*=(A^*)^\dag
A^*=V_1V_1^T$ as claimed.
\end{proof}

Corollary \ref{coro:favaro} offers us an insight into the relaxed R-LRR \eqref{equ:Favaro conclusion}: we can first solve the classical PCA problem with parameter $r=\arg\min_k k+\lambda\sum_{i>k}\sigma_i^2$ and then adopt the Shape Interaction Matrix of the denoised
data matrix as the affinity matrix for subspace clustering. This is consistent with the well-known fact that, empirically and theoretically, PCA is capable of effectively dealing with the small, dense Gaussian noises. Note that one needs to tune the parameter $\lambda$ in problem \eqref{equ:Favaro conclusion} in order to obtain a suitable parameter $r$ for the PCA problem.

\subsubsection{Other Cases}
\label{subsubsection: other cases}

Although our approach works well under rather broad conditions, as mentioned above, it might fail in some cases, e.g., the noiseless data matrix is not low rank. However, for certain data structure, the following numerical experiment shows that reducing to R-PCA correctly identifies the indices of noises even though the ground truth data matrix is of full rank\comment{\footnote{In this case, condition \eqref{equ: column-wise range} gives $m\le\rho_rm\mu^{-1}(\log n)^{-1}$. Notice that the constants $\rho_r$ and $\mu$ depend on the data distribution, and the above inequality is possible to hold when $\mu\log n\le \rho_r$.}}. The synthetic data are generated as follows. In the linear space
$\mathbb{R}^{5D}$, we construct five independent $D$ dimensional
subspaces $\{S_i\}_{i=1}^5$, whose bases $\{U_i\}_{i=1}^5$ are
randomly generated column orthonormal matrices. Then $20D$ points
are sampled from each subspace by multiplying its basis
matrix with a $D\times 20D$ Gaussian distribution matrix, whose
entries are i.i.d. $\mathcal{N}(0,1)$. Thus we obtain a $5D\times
100D$ structured sample matrix without noise, and the noiseless data matrix is of rank $5D$. We then add $15\%$ column-wise Gaussian noises whose entries are i.i.d. $\mathcal{N}(0,1)$ on the noiseless matrix, and solve model \eqref{equ: CSPCP} with $\lambda=1/\sqrt{\log (100D)}$. Table \ref{table: Exact Recovery for Random Problem of Varying Size} reports the Hamming distance between the ground truth indices and the identified indices by model \eqref{equ: CSPCP}, under different input sizes. It shows that reducing to R-PCA succeeds for structured data distributions even when the dimension of the sum of the subspaces is equal to that of the ambient space. In contrast, the algorithm fails for unstructured data distributions, e.g., the noiseless data is $5D\times100D$ Gaussian matrix whose element is totally random, obeying i.i.d. $\mathcal{N}(0,1)$. Since the main focus of the paper is the relations among several low rank models and the success conditions are within the research of R-PCA, the theoretical analysis on how data distribution influences the success of R-PCA will be our future work.
\makeatletter\def\@captype{table}\makeatother
\begin{table}
\caption{Exact identification of indices of noises on the matrix $M\in\mathbb{R}^{5D\times100D}$. Here
rank$(A_0)=5D$, $||E_0||_{2,0}=15D$, and $\lambda=1/\sqrt{\log (100D)}$. $\mathcal{I}^*$ refers to the indices obtained by solving model \eqref{equ: CSPCP} and $\mathcal{I}_0$ refers to the ground truth indices of noises.}
\label{table: Exact Recovery for Random Problem of Varying Size}
\begin{center}
\begin{tabular}{|c|c||c|c||c|c||c|c|}
\hline
$D$ & dist$(\mathcal{I}^*,\mathcal{I}_0)$ & $D$ & dist$(\mathcal{I}^*,\mathcal{I}_0)$ & $D$ & dist$(\mathcal{I}^*,\mathcal{I}_0)$ & $D$ & dist$(\mathcal{I}^*,\mathcal{I}_0)$\\
\hline
5 & 0 & 10 & 0 & 50 & 0 & 100 & 0\\
\hline
\end{tabular}
\end{center}
\end{table}

\subsection{Fast Algorithms for Subspace Recovery}\label{sec:fast_algs}
Representative low rank subspace recovery models, like LRR and
LatLRR, are solved by ADM~\citep{Lin2} and the complexity is
$O(mn^2)$~\citep{LiuG2,LiuG3,LiuG1}. For LRR, by employing
linearized ADM (LADM) and some advanced tricks for computing
partial SVD, the resulted algorithm is of $O(rn^2)$ complexity,
where $r$ is the rank of optimal $Z$. We show that our REDU-EXPR
approach can be much faster.

We take a real experiment for an example. We test face image
clustering on the extended YaleB database, which consists of 38
persons with 64 different illuminations for each person. All the
faces are frontal and thus images of each person lie in a low
dimensional subspace \citep{Belhumeur:Fisherface}. We generate the input data as follows.
We reshape each image into a 32,256 dimensional column vector.
Then the data matrix $X$ is $32,256\times 2,432$. We record the
running times and the clustering accuracies\footnote{\cite{LiuG2}
reported an accuracy of 62.53\% by LRR, but there were only 10
classes in their data set. In contrast, there are 38 classes in
our data set.} of relaxed LRR~\citep{LiuG2,LiuG1} and relaxed
R-LRR~\citep{Favaro,Wei}. LRR is solved by ADM. For R-LRR we test
three algorithms. The first one is traditional ADM, i.e., updating
$A$, $E$, and $Z$ alternately by minimizing the augmented
Lagrangian function of relaxed R-LRR:
\begin{equation}
L(A,E,Z)=||Z||_*+\lambda f(E)+\langle X-E-A, Y_1\rangle+\langle A-AZ, Y_2\rangle+\frac{\mu}{2}(||X-E-A||_F^2+||A-AZ||_F^2).
\end{equation}
The second algorithm is partial ADM, which updates $A$, $E$, and
$Z$ by minimizing the partial augmented Lagrangian function:
\begin{equation}
L(A,E,Z)=||Z||_*+\lambda f(E)+\langle X-E-A, Y\rangle+\frac{\mu}{2}||X-E-A||_F^2,
\end{equation}
subject to $A=AZ$. This method is adopted by \cite{Favaro}. A key
difference between partial ADM and traditional ADM is that the former updates $A$ and
$Z$ simultaneously by utilizing
Corollary~\ref{coro:favaro}. For more details, please refer to
\citep{Favaro}. The third method is REDU-EXPR. It is adopted by
\cite{Wei}. Except the ADM method for solving R-LRR, we run the
codes provided by their respective authors.

One can see from Table~\ref{tab:YaleB} that REDU-EXPR is
significantly faster than ADM based method. Actually, solving
R-LRR by ADM did not converge. We want to point out that the
partial ADM method utilized the closed-form solution shown in
Corollary~\ref{coro:favaro}. However, its speed is still much inferior to
that of REDU-EXPR.
\makeatletter\def\@captype{table}\makeatother
\begin{table}
\caption{Unsupervised face image clustering results on the
Extended YaleB database.} \label{tab:YaleB}
\begin{center}
\begin{tabular}{|c|c||c|c|}
\hline
Model & Method & Accuracy & CPU Time (h)\\
\hline\hline
LRR & ADM & - & $>$10\\
R-LRR & ADM & - & did not converge\\
R-LRR & partial ADM & - & $>$10\\
R-LRR & REDU-EXPR & 61.6365\% & 0.4603\\
\hline
\end{tabular}
\end{center}
\end{table}

\comment{ Furthermore, there are still spaces for further speeding
up the efficiency by designing more effective PCP algorithms. The
currently existed algorithms mainly include the dual
method~\citep{Ganesh}, the accelerated proximal gradient
method~\citep{Lin1} and the ADM~\citep{Lin2}. However, all these
methods require singular value decomposition (SVD) as a tool for
solving subproblem
\begin{equation}
\min_A \eta||A||_*+\frac{1}{2}||A-W||_F^2.
\end{equation}
So the complexity is still high for very large data sets. A more
recent proposed technique, termed $\ell_1$
filtering~\citep{LiuR1}, first solves a small scale PCP problem
via ADM and then uses linear representation to get the whole
solution. The complexity was reported as $O(r^2(m+n))$. Here we
review some of main techniques appeared in $\ell_1$
filtering~\citep{LiuR1}, and then, by a little modification, we
propose our $\ell_{2,1}$ filtering algorithm on the $\ell_{2,1}$
norm. }

For large scale data, neither $O(mn^2)$ nor $O(rn^2)$ is fast
enough. Fortunately, for R-PCA it is relatively easily to design
low complexity randomized algorithms to further reduce its
computational load. \cite{LiuR1} has reported an efficient
randomized algorithm called $\ell_1$ filtering to solve R-PCA when
$f(E)=\|E\|_{\ell_1}$. The $\ell_1$ filtering is completely parallel and
its complexity is only $O(r^2(m+n))$ -- linear to the matrix
size. In the following, we sketch the $\ell_1$ filtering
algorithm~\citep{LiuR1}, and in the same spirit propose a novel
$\ell_{2,1}$ filtering algorithm for solving column sparse R-PCA
\eqref{equ: CSRPCA}, i.e., R-PCA with $f(E)=\|E\|_{2,1}$.

\subsubsection{Outline of $\ell_1$ Filtering Algorithm~\citep{LiuR1}}
The $\ell_1$ filtering algorithm aims at solving the R-PCA problem
\eqref{equ: generalized RPCA} with $f(E)=\|E\|_1$. There are two
main steps. The first step is to recover a seed matrix. The second
step is to process the rest part of data matrix by $\ell_1$-norm
based linear regression.

\paragraph{Recovery of a Seed Matrix}
Assume that the target rank $r$ of the low rank component $A$ is
very small compared with the size of the data matrix, i.e., $r\ll
\min\{m,n\}$. By randomly sampling an $s_rr\times s_cr$ submatrix
$X^s$ from $X$, where $s_r,s_c>1$ are oversampling rates, we
partition the data matrix $X$, together with the underlying matrix
$A$ and the noise $E$, into four parts (for simplicity we assume
that $X^s$ is at the top left corner of $X$):
\begin{equation}
\label{equ:complete A}
X=\begin{bmatrix}
X^s & X^c\\
X^r & \widetilde{X}^s
\end{bmatrix},\quad
A=\begin{bmatrix}
A^s & A^c\\
A^r & \widetilde{A}^s
\end{bmatrix},\quad
E=\begin{bmatrix}
E^s & E^c\\
E^r & \widetilde{E}^s
\end{bmatrix}.
\end{equation}

We firstly recover the seed matrix $A^s$ of the underlying matrix
$A$ from $X^s$ by solving a small scale relaxed R-PCA problem:
\begin{equation}
\min_{A^s,E^s} ||A^s||_*+\lambda^s||E^s||_{\ell_1}, \ \
\mbox{s.t.} \ \ X^s=A^s+E^s,
\end{equation}
where $\lambda^s=1/\sqrt{\max\{s_rr,s_cr\}}$ which is suggested in
\citep{Candes} for exact recovery of the underlying $A^s$. This
problem can be efficiently solved by ADM~\citep{Lin2}.

\comment{ According to Theorem 1.1 in \cite{Candes}, it is
reasonable to believe that $A^s$ could be exactly recovered as the
rank of the low rank component is not too large and the sparse
component is sparse.}

\paragraph{$\ell_1$ Filtering}

Since rank$(A)=r$ and $A^s$ is a randomly sampled $s_r r\times s_c
r$ submatrix of $A$, with an overwhelming probability
rank$(A^s)=r$. So $A^c$ and $A^r$ must be represented as linear
combinations of the columns or rows in $A^s$. Thus we obtain the
following $\ell_1$-norm based linear regression problems:
\begin{equation}
\min_{Q,E^c} ||E^c||_{\ell_1},\ \ \mbox{s.t.}\ \ X^c=A^sQ+E^c,
\end{equation}
\begin{equation}
\min_{P,E^r} ||E^r||_{\ell_1},\ \ \mbox{s.t.}\ \ X^r=P^TA^s+E^r.
\end{equation}

As soon as $A^c=A^sQ$ and $A^r=P^TA^s$ are computed, the
generalized Nystr$\ddot{\mbox{o}}$m method~\citep{WangJ} gives
\begin{equation}
\widetilde{A}^s=P^T A^s Q.
\end{equation}
Thus we recover all the submatrices in $A$. As shown in
\citep{LiuR1}, the complexity of this algorithm is only
$O(r^2(m+n))$ without considering the reading and writing time.

\subsubsection{$\ell_{2,1}$ Filtering Algorithm}
$\ell_1$ filtering is for entry sparse R-PCA. For R-LRR, we need to
solve column sparse R-PCA. Unlike the $\ell_1$
case which breaks the whole matrix into four blocks, the
$\ell_{2,1}$ norm requires viewing each column in a holistic way.
So we can only partition the whole matrix into two blocks. We
inherit the idea of $\ell_1$ filtering to propose a randomized
algorithm, called $\ell_{2,1}$ filtering, to solve column sparse
R-PCA. It also consists of two steps. We first
recover a seed matrix and then process the remaining columns via
$\ell_{2}$ norm based linear regression, which turns out to be a
least square problem.

\paragraph{Recovery of a Seed Matrix}
The step of recovering a seed matrix is nearly the same as that of
the $\ell_1$ filtering method, except that we only partition the
whole matrix into two blocks. Suppose the rank of $A$ is $r\ll
\min\{m,n\}$. We randomly sample $sr$ columns of $X$, where $s>1$
is an oversampling rate. These $sr$ columns form a submatrix
$X_l$. For brevity, we assume that $X_l$ is the leftmost submatrix
of $X$. Then we may partition $X$, $A$, and $E$ as follows:
$$X=[X_l,X_r],\quad E=[E_l,E_r],\quad A=[A_l,A_r],$$
respectively. We could firstly recover $A_l$ from $X_l$ by a small
scale relaxed column sparse R-PCA problem:
\begin{equation}
\label{equ:recovery of the seed matrix} \min_{A_l,E_l}
||A_l||_*+\lambda_l||E_l||_{\ell_{2,1}}, \ \ \mbox{s.t.} \ \
X_l=A_l+E_l,
\end{equation}
where $\lambda_l=1/\sqrt{\log (sr)}$~\citep{Zhang2014NIPS}.

\paragraph{$\ell_{2,1}$ Filtering}
After the seed matrix $A_l$ is obtained, since rank$(A)=r$ and
with an overwhelming probability rank$(A_l)=r$, the columns of
$A_r$ must be linear combinations of $A_l$. So there exists a
representation matrix $Q\in\mathbb{R}^{sr\times (n-sr)}$ such that
\begin{equation}
A_r=A_lQ.
\end{equation}
On the other hand, the part $E_r$ of noise should still be column
sparse. So we have the following $\ell_{2,1}$ norm based linear
regression problem:
\begin{equation}
\label{equ:l21 filtering}
\min_{Q,E_r} ||E_r||_{\ell_{2,1}}, \ \ \mbox{s.t.} \ \ X_r=A_lQ+E_r.
\end{equation}

If \eqref{equ:l21 filtering} is solved directly by using
ADM~\citep{LiuR2}, the complexity of our algorithm will be nearly
the same as that of solving the whole original problem.
Fortunately, we can solve \eqref{equ:l21 filtering} column-wise
independently due to the separability of $\ell_{2,1}$ norms.

Let $x_r^{(i)}$, $q^{(i)}$, and $e_r^{(i)}$ represent the $i$th
column of $X_r$, $Q$, and $E_r$, respectively
($i=1,2,\cdots,n-sr$). Then problem \eqref{equ:l21 filtering}
could be decomposed into $n-sr$ subproblems:
\begin{equation}
\label{equ:least-square problem} \min_{q^{(i)},e_r^{(i)}}
||e_r^{(i)}||_2, \ \ \mbox{s.t.} \ \
x_r^{(i)}=A_lq^{(i)}+e_r^{(i)},\ i=1,\cdots,n-sr.
\end{equation}
As least square problems, \eqref{equ:least-square problem} has
closed-form solutions $q^{(i)}=A_l^\dag x_r^{(i)},
i=1,\cdots,n-sr$. Then $Q^*=A_l^\dag X_l$ and the solution to the
original problem \eqref{equ:l21 filtering} is $(A_l^\dag
X_r,X_r-A_lA_l^\dag X_r)$. Interestingly, it is the same solution
if replacing the $\ell_{2,1}$ norm in \eqref{equ:l21 filtering}
with the Frobenius norm.

\comment{
\begin{equation}
\begin{split}
&\min_{q^{(i)},e_r^{(i)}} \sum_{i=1}^{n-s_cr}||e_r^{(i)}||_2,\\
&\ \ \mbox{s.t.} \ \ x_r^{(i)}=A_lq^{(i)}+e_r^{(i)},\\
&\ \ \ \ \ \ \ \ i=1,\cdots,n-sr.
\end{split}
\end{equation}
Notice that the $n-s_cr$ constraints are independent of each
other. Hence, by adopting a subproblem strategy via making the
objective on every column to be minimum, we can accurately get the
solution of $i$th column via simply optimizing }

Note that our target is to recover the right patch $A_r=A_lQ^*$.
Let $U_{A_l}\Sigma_{A_l} V_{A_l}^T$ be the skinny SVD of $A_l$,
which is available when solving \eqref{equ:recovery of the seed
matrix}. Then $A_r$ could be written as
\begin{equation}
A_r=A_lQ^*=A_lA_l^\dag X_r=U_{A_l}U_{A_l}^TX_r.
\end{equation}
We may first compute $U_{A_l}^TX_r$ and then
$U_{A_l}(U_{A_l}^TX_r)$. This little trick reduces the complexity
of computing $A_r$.

\paragraph{The Complete Algorithm}
Algorithm~\ref{alg:l_21 filtering} summarizes our $\ell_{2,1}$
filtering algorithm for solving column sparse R-PCA.

\begin{algorithm}[h]
\caption{$\ell_{2,1}$ Filtering Algorithm for Column Sparse
R-PCA}\label{alg:l_21 filtering}
\begin{algorithmic}
\STATE {\bfseries Input:} Observed data matrix $X$ and estimated
rank $r$.
\STATE {\bfseries 1.} Randomly sample $sr$ columns from $X$ to form $X_l$.
\STATE {\bfseries 2.} Solve small scale relaxed R-PCA \eqref{equ:recovery of the seed matrix} by ADM and
obtain SVD of $A_l$: $U_{A_l}\Sigma_{A_l} V_{A_l}^T$.\vspace{-1cm}
\STATE {\bfseries 3.} Recover $A_r=A_lQ$ by solving \eqref{equ:l21
filtering}, whose solution is $A_r=U_{A_l}(U_{A_l}^TX_r)$.
\STATE {\bfseries Output:} Low rank component $A=[A_l,A_r]$ and column sparse matrix $E=X-A$.
\end{algorithmic}
\end{algorithm}
As soon as the solution $(A,E)$ to column sparse R-PCA is solved,
we can obtain the representation matrix of R-LRR $Z$ by $Z=A^\dag
A$. Note that we should not compute $Z$ naively as it is written,
whose complexity will be more than $O(mn^2)$. A more clever way is
as follows. Suppose $U_A\Sigma_AV_A^T$ is the skinny SVD of $A$,
then $Z=A^\dag A=V_AV_A^T$. On the other hand,
$A=U_{A_l}[\Sigma_{A_l}V_{A_l}^T,U_{A_l}^TX_r]$. So we only have
to compute the row space of
$\hat{A}=[\Sigma_{A_l}V_{A_l}^T,U_{A_l}^TX_r]$, where
$U_{A_l}^TX_r$ has been saved in Step 3 of Algorithm~\ref{alg:l_21
filtering}. This can be easily done by doing LQ decomposition~\citep{golub2012matrix} of
$\hat{A}$: $\hat{A}=LV^T$, where $L$ is lower triangular and
$V^TV=I$. Then $Z=VV^T$. Since LQ decomposition is much cheaper
than SVD, the above trick is very efficient and all the
matrix-matrix multiplications are $O(r^2n)$. The complete
procedure for solving R-LRR problem \eqref{equ:relaxed RSI} is
described in Algorithm \ref{algorithm: Fast RSI}.

\begin{algorithm}[h]
\caption{Subspace Clustering Based on the relaxed R-LRR Model
\eqref{equ:relaxed RSI}}
\begin{algorithmic}
\STATE {\bfseries Input:} Observed data matrix $X$, estimated rank
$r$. \STATE {\bfseries 1.} Solve relaxed column sparse R-PCA
\eqref{equ: generalized PCP} with
$f(E)=\|E\|_{\ell_{2,1}}$ by Algorithm~\ref{alg:l_21 filtering}.
\STATE {\bfseries 2.} Conduct LQ decomposition on the matrix
$\hat{A}=[\Sigma_{A_l}V_{A_l}^T,U_{A_l}^TX_r]$ as $\hat{A}=LV^T$.
\STATE {\bfseries 3.} Get the affinity matrix by $|Z|=|VV^T|$ and
conduct spectral clustering. \STATE {\bfseries Output:} Label for
each data point.
\end{algorithmic}
\label{algorithm: Fast RSI}
\end{algorithm}

Unlike LRR, the optimal solution to R-LRR problem
\eqref{equ:relaxed RSI} is symmetric and thus we could directly
use $|Z|$ as the affinity matrix instead of $|Z|+|Z^T|$. After
that, we can apply spectral clustering algorithms, such as
Normalized Cut, to cluster each data point into its corresponding
subspace.

\paragraph{Complexity Analysis}
In Algorithm \ref{alg:l_21 filtering}, Step 2 requires $O(r^2m)$
time and Step 3 requires $2rmn$ time. Thus the whole complexity of
the $\ell_{2,1}$ filtering algorithm for solving column sparse
R-PCA is $O(r^2m)+2rmn$. In Algorithm \ref{algorithm: Fast RSI}
for solving the relaxed R-LRR problem \eqref{equ:relaxed RSI},
as just analyzed, Step 1 requires $O(r^2m)+2rmn$ time. The LQ
decomposition in Step 2 requires $6r^2n$ time at the most \citep{golub2012matrix}.
Computing $VV^T$ in Step 3 requires $rn^2$ time. Thus the whole
complexity for solving \eqref{equ:relaxed RSI} is
$O(r^2m)+6r^2n+2rmn+rn^2$\footnote{Here we want to highlight the difference
between $2rmn+rn^2$ and $O(rmn+rn^2)$. The former is independent
of numerical precision. It is due to the three matrix-matrix
multiplications to form $\hat{A}$ and $Z$, respectively. In
contrast, $O(rmn+rn^2)$ usually grows with the numerical
precision. The more iterations are, the larger constant in the big
$O$ is.}. As most of low rank subspace clustering models require
$O(mn^2)$ time to solve, due to SVD or matrix-matrix
multiplication in every iteration, our algorithm is significantly
faster than state-of-the-art methods.

\section{Experiments}
\label{section:experiments}

In this section, we use experiments to illustrate the applications of our theoretical analysis.

\subsection{Comparison of Optimality on Synthetic Data}
\label{subsection:comparison of optimality on synthetic data} In
this subsection, we compare the two algorithms, partial ADM\footnote{The partial ADM method of \cite{Favaro} was designed for the $\ell_1$ norm on the noise matrix $E$, while here we have adapted it for the $\ell_{2,1}$ norm.}
~\citep{Favaro} and REDU-EXPR~\citep{Wei}, which we have mentioned in
Section~\ref{sec:fast_algs}, for solving the non-convex relaxed R-LRR
problem \eqref{equ:relaxed RSI}. Since the
traditional ADM is not-convergent, we do not compare with it.
Because we only want to compare the quality of solutions produced by
the two methods, for REDU-EXPR we temporarily do not use the $\ell_{2,1}$
filtering algorithm introduced in Section
\ref{section:applications of theoretical analysis} to solve column
sparse R-PCA.

The synthetic data are generated as follows. In the linear space
$\mathbb{R}^{1000}$, we construct five independent four dimensional
subspaces $\{S_i\}_{i=1}^5$, whose bases $\{U_i\}_{i=1}^5$ are
randomly generated column orthonormal matrices. Then 200 points
are uniformly sampled from each subspace by multiplying its basis
matrix with a $4\times 200$ Gaussian distribution matrix, whose
entries are i.i.d. $\mathcal{N}(0,1)$. Thus we obtain a $1,000\times
1,000$ sample matrix without noise.

We compare the clustering accuracies\footnote{Just as \cite{LiuG2}
did, given the ground truth labeling we set the label of a cluster
to be the index of the ground truth which contributes the maximum
number of the samples to the cluster. Then all these labels are
used to compute the clustering accuracy after comparing with the
ground truth.} as the percentage of corruptions increases, where
noises uniformly distributed on $(-0.6,0.6)$ are added at
uniformly distributed positions. We run the test ten times and
compute the mean clustering accuracy. Figure~\ref{figure:
comparison of robustness} presents the comparison on the accuracy,
where all the parameters are tuned to be the same, i.e., $\lambda=1/\sqrt{\log 1000}$. One can see that R-LRR solved by REDU-EXPR is much
more robust to column sparse corruptions than by partial ADM.

\begin{figure}[h]
\centering
\includegraphics[width=0.6\textwidth]{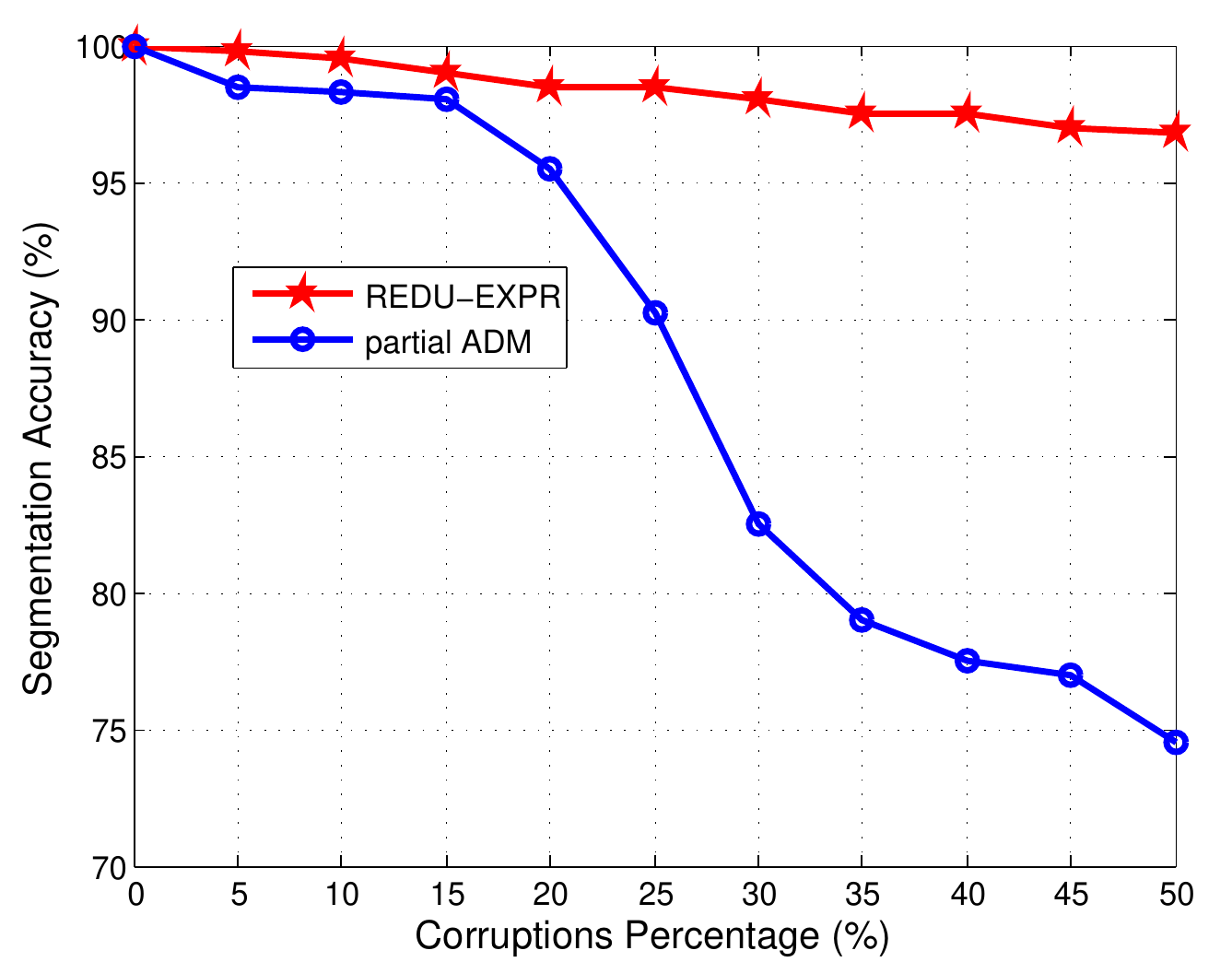}
\caption{Comparison of accuracies of solutions to relaxed
R-LRR~\eqref{equ:relaxed RSI} computed by REDU-EXPR \citep{Wei} and
partial ADM \citep{Favaro}, where the parameter $\lambda$ is
adopted as $1/\sqrt{\log n}$ and $n$ is the input size. The program is run by 10 times and the
average accuracies are reported.} \label{figure: comparison of
robustness}
\end{figure}

To further compare the optimality, we also record the objective function values computed by the two algorithms.
Since both algorithms aim at achieving the low rankness of the affinity matrix and the column sparsity of the noise matrix,
we compare the objective function of the original R-LRR \eqref{equ:original RSI}, i.e.,
\begin{equation}
\mathcal{F}(Z,E)=\mbox{rank}(Z)+\lambda||E||_{\ell_{2,0}}.
\end{equation}
As shown in Table \ref{table: comparison of optimality}, R-LRR by
REDU-EXPR could obtain smaller rank$(Z)$ and objective function than those of partial ADM. Table \ref{table: comparison of
optimality} also shows the CPU times (in seconds). One can see
that REDU-EXPR is significantly faster than partial ADM when solving the same model.

\makeatletter\def\@captype{table}\makeatother
\begin{table}[h]
\caption{Comparison of robustness and speed between partial
ADM (LRSC)~\citep{Favaro} and REDU-EXPR (RSI)~\citep{Wei} methods for solving
R-LRR when the percentage of corruptions increases. All the
experiments are run ten times and the $\lambda$ is set to be the
same: $\lambda=1/\sqrt{\log n}$, where $n$ is the data size.} \label{table: comparison of optimality}
\begin{center}
\begin{tabular}{|c|cccccc|}
\hline
Noise Percentage (\%) & 0 & 10 & 20 & 30 & 40 & 50\\
\hline\hline
Rank($Z$) (partial ADM) & 20 & 30 & 30 & 30 & 30 & 30\\
Rank($Z$) (REDU-EXPR) & 20 & 20 & 20 & 20 & 20 & 20\\
\hline
$||E||_{\ell_{2,0}}$ (partial ADM) & 0 & 99 & 200 & 300 & 400 & 500\\
$||E||_{\ell_{2,0}}$ (REDU-EXPR) & 0 & 100 & 200 & 300 & 400 & 500\\
\hline
Objective (partial ADM) & \textbf{20.00} & 67.67 & 106.10 & 144.14 & 182.19 & 220.24\\
Objective (REDU-EXPR) & \textbf{20.00} & \textbf{58.05} & \textbf{96.10} & \textbf{134.14} & \textbf{172.19} & \textbf{210.24}\\
\hline
Time (s, partial ADM) & \textbf{4.89} & 124.33 & 126.34 & 119.12 & 115.20 & 113.94\\
Time (s, REDU-EXPR) & 10.67 & \textbf{9.60} & \textbf{8.34} & \textbf{8.60} & \textbf{9.00} & \textbf{12.86}\\
\hline
\end{tabular}
\end{center}
\end{table}

\subsection{Comparison of Speed on Synthetic Data}
In this subsection, we show the great speed advantage of our
REDU-EXPR algorithm in solving low rank recovery models. We compare
the algorithms to solve relaxed R-LRR \eqref{equ:relaxed RSI}.
We also present the results of solving LRR by ADM for reference,
although it is a slightly different model. Except our $\ell_{2,1}$
filtering algorithm, all the codes run in this test are offered by
the authors of \cite{LiuG1}, \cite{LiuG3}, and \cite{Favaro}.

\comment{ Another application of our theory is, as
LRR~\citep{LiuG2,LiuG1}, LatLRR~\citep{LiuG3}, LRSC~\citep{Favaro}
and RSI~\citep{Wei} all require high complexity for computation,
by using our fast algorithm Fast RSI, the complexity of subspace
clustering could be reduced to nearly linear with respect to the
data size $n$, no matter which sparsity norm we choose.

For fair comparison, all the models in this test\footnote{Except
our fast algorithm, all the codes run in this test are offered by
the authors of \cite{LiuG1,LiuG3,Favaro}} use the $\ell_{2,1}$
norm as the sparsity norm except LatLRR, since LatLRR only makes
sense in the meaning of the $\ell_1$ norm.}

The parameter $\lambda$ is set for each method so that the highest
accuracy is obtained. We generate clean data as we did in Section
\ref{subsection:comparison of optimality on synthetic data}. The
only differences are the choice of the dimension of the ambient
space and the number of points sampled from subspaces. We compare
the speed of different algorithms on corrupted data, where the
noises are added in the same way as in \citep{LiuG2} and
\citep{LiuG1}. Namely, the noises are added by submitting to 5\%
column-wise Gaussian noises with zero means and $0.1||x||_2$
standard deviation, where $x$ indicates corresponding vector in
the subspace. For REDU-EXPR, with or without using $\ell_{2,1}$
filtering, the rank is estimated at its exact value, twenty, and
the over-sampling parameter $s_c$ is set to be ten. As the data
size goes up, the CPU times are shown in Table
\ref{table:comparison of speed}. When the corruptions are not heavy,
all the methods in this test achieve 100\% accuracy. We can see
that REDU-EXPR consistently outperforms ADM based methods. By
$\ell_{2,1}$ filtering, the computation time is further
reduced. The advantage of $\ell_{2,1}$ filtering is more salient
when the data size is larger.
\makeatletter\def\@captype{table}\makeatother
\begin{table}[h]
\caption{Comparison of CPU time (seconds) between
LRR~\citep{LiuG2,LiuG1} solved by ADM, R-LRR solved by partial
ADM (LRSC)~\citep{Favaro}, R-LRR solved by REDU-EXPR without using
$\ell_{2,1}$ filtering (RSI)~\citep{Wei}, and R-LRR solved by REDU-EXPR
using $\ell_{2,1}$ filtering as the data size increases. In this
test, REDU-EXPR with $\ell_{2,1}$ filtering is significantly
faster than other methods and its computation time grows at most
linearly with the data size.} \label{table:comparison of speed}
\begin{center}
\begin{tabular}{|c|cccc|}
\hline
Data Size & LRR  & R-LRR & R-LRR & R-LRR\\
& (ADM) & (partial ADM) & (REDU-EXPR) & (filtering REDU-EXPR)\\
 \hline\hline
$250\times250$ & 33.0879 & 4.9581 & 1.4315 & \textbf{0.6843}\\
$500\times500$ & 58.9177 & 7.2029 & 1.8383 & \textbf{1.0917}\\
$1,000\times1,000$ & 370.1058 & 24.5236 & 6.1054 & \textbf{1.5429}\\
$2,000\times2,000$ & $>$3600 & 124.3417 & 28.3048 & \textbf{2.4426}\\
$4,000\times4,000$ & $>$3600 & 411.8664 & 115.7095 & \textbf{3.4253}\\
\hline
\end{tabular}
\end{center}
\end{table}

\subsection{Test on Real Data -- AR Face Database}

Now we test different algorithms on real data, the AR Face
database, to classify face images. The AR face database contains
2,574 color images of 99 frontal faces. All the faces are with
different facial expressions, illumination conditions, and
occlusions (e.g. sun glasses or scarf, see Figure \ref{figure: AR
database}), thus the AR database is much harder than the YaleB
database for face clustering. So we replace the spectral
clustering (Step 3 in Algorithm \ref{algorithm: Fast RSI}) with a
linear classifier. The classification is done as follows:
\begin{equation}
\min_{W} ||H-WF||_F^2+\gamma||W||_F^2,
\end{equation}
which is simply a ridge regression and the regularization
parameter $\gamma$ is fixed at 0.8, where $F$ is the feature data and $H$
is the label matrix. The classifier is trained as follows.
We first run LRR or R-LRR on the original input data $X\in\mathbb{R}^{m\times n}$
and obtain an approximately block diagonal matrix $Z\in\mathbb{R}^{n\times n}$.
View each column of $Z$ as a new observation\footnote{Since $Z$ is approximately block diagonal, each column of $Z$ has few non-zero coefficients and thus the new observations are suitable for classification.}
and separate the columns of $Z$ into two parts,
where one part corresponds to the training data and the other corresponds to the test data.
We train the ridge regression model by the training samples and use the obtained $W$ to classify the test samples.

\begin{figure}[h]
\centering
\subfigure{
\includegraphics[width=0.13\textwidth]{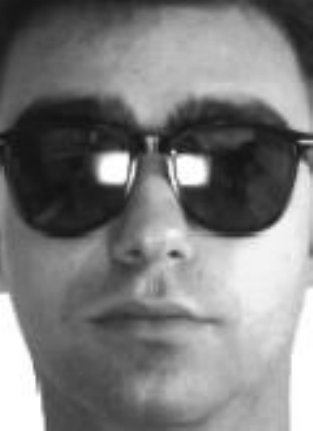}}
\subfigure{
\includegraphics[width=0.13\textwidth]{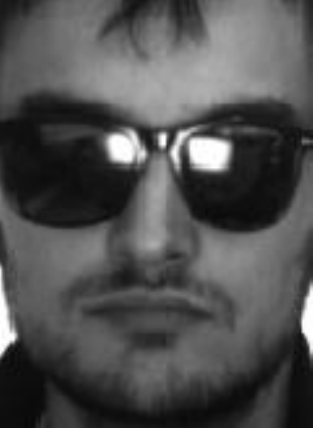}}
\subfigure{
\includegraphics[width=0.13\textwidth]{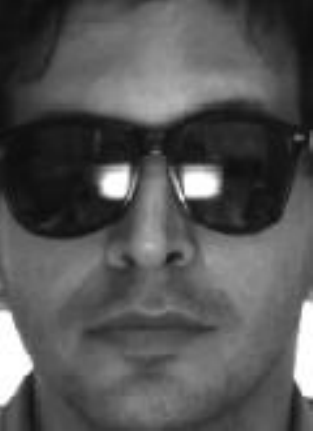}}
\subfigure{
\includegraphics[width=0.13\textwidth]{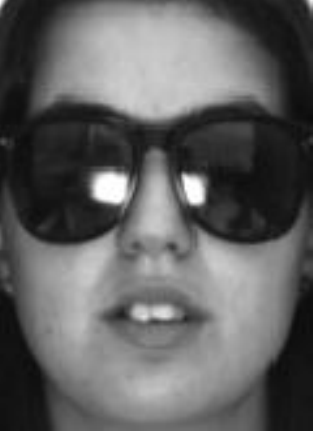}}
\subfigure{
\includegraphics[width=0.13\textwidth]{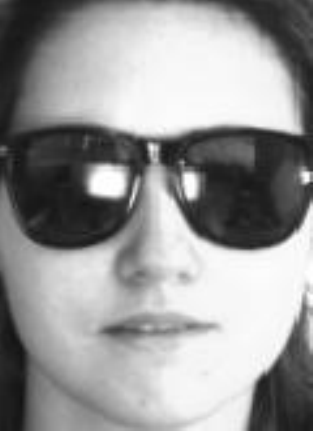}}
\subfigure{
\includegraphics[width=0.13\textwidth]{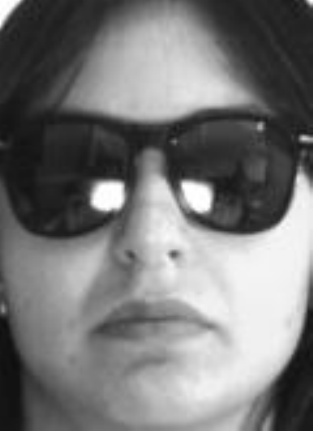}}
\subfigure{
\includegraphics[width=0.13\textwidth]{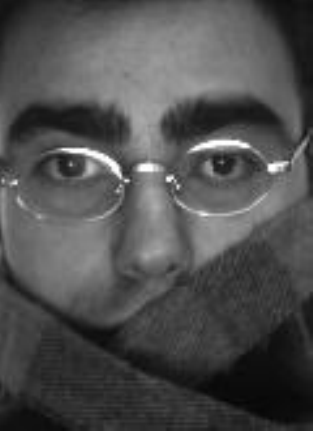}}
\subfigure{
\includegraphics[width=0.13\textwidth]{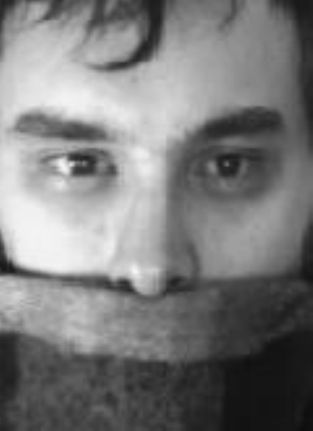}}
\subfigure{
\includegraphics[width=0.13\textwidth]{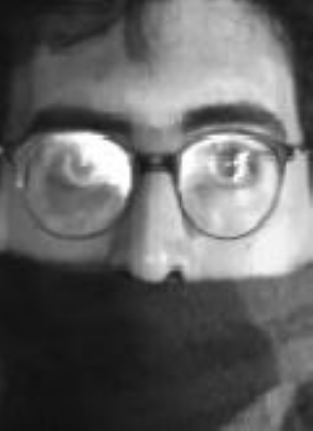}}
\subfigure{
\includegraphics[width=0.13\textwidth]{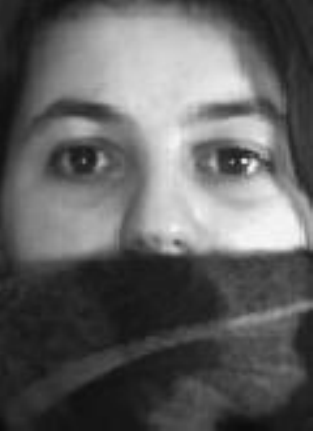}}
\subfigure{
\includegraphics[width=0.13\textwidth]{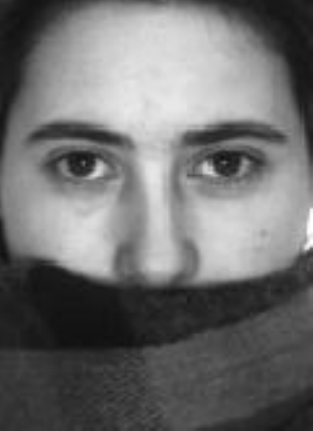}}
\subfigure{
\includegraphics[width=0.13\textwidth]{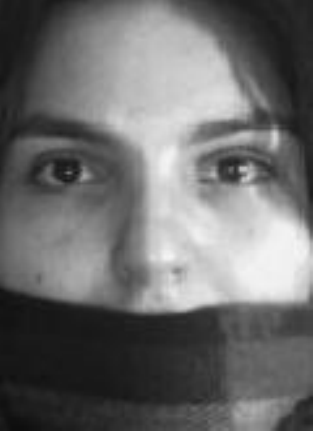}}
\caption{Examples of images with severe occlusions in the AR
database. The images in the same column belong to the same
person.} \label{figure: AR database}
\end{figure}

%
%

Unlike the existing literatures, e.g., \cite{LiuG2,LiuG1}, which
manually removed severely corrupted images and shrank the input
images to small-sized ones in order to reduce the computation
load, our experiment uses all the \emph{full-sized} face images. So the
size of our data matrix is $19,800\times 2,574$, where each image
is reshaped as a column of the matrix, 19,800 is the number of
pixels in each image, and 2,574 is the total number of face
images. We test LRR~\citep{LiuG2,LiuG1} (solved by ADM) and
relaxed R-LRR (solved by partial ADM~\citep{Favaro},
REDU-EXPR~\citep{Wei}, and REDU-EXPR with $\ell_{2,1}$ filtering)
for both classification accuracy and speed. Table \ref{table: test on AR} shows the results, where the parameters have been
tuned to be the best. Since ADM based method requires too long
time to converge, we terminate it after sixty hours. This
experiment testifies to the great speed advantage of REDU-EXPR and
$\ell_{2,1}$ filtering. Note that with $\ell_{2,1}$ filtering the
speed of REDU-EXPR is three times faster than that without $\ell_{2,1}$
filtering, and the accuracy is not compromised.
\makeatletter\def\@captype{table}\makeatother
\begin{table}
\caption{Comparison of classification accuracy and speed on the AR
database with the task of face image classification. For fair comparison
of both the accuracy and the speed for different algorithms, the
parameters are tuned to be the best according to the
classification accuracy and we observe the CPU time.} \label{table: test on AR}
\begin{center}
\begin{tabular}{|c|c|cc|}
\hline
Model & Method & Accuracy & CPU Time (h)\\
\hline\hline
LRR & ADM & - & $>$10\\
R-LRR & partial ADM  & - & $>$10\\
R-LRR  & REDU-EXPR & 90.1648\% & 0.5639\\
R-LRR  & REDU-EXPR with $\ell_{2,1}$ filtering & \textbf{90.5901\%} & \textbf{0.1542}\\
\hline
\end{tabular}
\end{center}
\end{table}

\section*{Conclusion and Future Work}
\label{section:conclusion} In this paper, we investigate the
connections among solutions of some representative low rank
subspace recovery models, including R-PCA, R-LRR, R-LatLRR, and
their convex relaxations. We show that their solutions can be
mutually expressed in closed forms. Since R-PCA is the simplest
model, it naturally becomes a hinge to all low rank subspace
recovery models. Based on our theoretical findings, under certain conditions we are able to
find better solutions to low rank subspace recovery models and
also significantly speed up finding their solutions numerically,
by solving R-PCA first and then express their solutions by that of
R-PCA in closed forms. Since there are randomized algorithms for
R-PCA, e.g., we propose the $\ell_{2,1}$ filtering algorithm for
column sparse R-PCA, the computation complexities for solving
existing low rank subspace recovery models can be much lower than the existing algorithms. Extensive experiments on both synthetic and real world
data testify to the utility of our theories.

As shown in Section \ref{subsubsection: other cases}, our approach may succeed even when the conditions of Section \ref{subsubsection: Sparse Element-wise Noises}, \ref{subsubsection: Sparse Column-wise Noises}, and \ref{subsubsection: Dense Gaussian Noises} do not hold. The theoretical analysis on how data distribution influences the success of our approach will be
our future work.

\subsection*{Acknowledgments}
The authors thank Rene Vidal for valuable discussions. Hongyang Zhang and Chao Zhang are supported by National Key Basic Research Project of China (973 Program) 2011CB302400 and National Nature Science Foundation of China (NSFC grant, no. 61071156). Zhouchen Lin is supported by 973 Program of China (grant no. 2015CB3525), NSF China (grant nos. 61272341 and 61231002), and Microsoft Research Asia Collaborative Research Program. Junbin Gao is supported by the Australian Research Council (ARC) through the grant DP130100364.

\bibliographystyle{apa}
\bibliography{reference}

\end{document}